\documentclass[10pt]{article}

\usepackage[preprint]{tmlr}

\usepackage[utf8]{inputenc}
\usepackage{microtype}
\usepackage{graphicx}
\usepackage{subcaption}
\usepackage{booktabs} 
\usepackage{url}
\usepackage{nicefrac}
\usepackage{hyperref}

\usepackage{placeins}
\usepackage{apxproof}
\usepackage{amsmath}
\usepackage{amssymb}
\usepackage{mathtools}
\usepackage{amsthm}

\usepackage{amsfonts}  
\usepackage{commath}   
\usepackage{cancel}    
\usepackage{bm}        
\usepackage{xcolor}    
\usepackage{etoolbox}  
\usepackage{comment}   
\usepackage{mdframed}  
\usepackage{xspace}    
\usepackage{siunitx}   
\usepackage{float}     

\usepackage{doi}          

\DeclareMathOperator{\diag}{\operatorname{diag}}

\usepackage[capitalize,noabbrev]{cleveref}

\theoremstyle{plain}
\newtheorem*{informal*}{Informal Theorem}
\newtheorem{intuition}{Intuition}
\newtheorem{insight}{Insight}
\newtheorem{theorem}{Theorem}[section]

\theoremstyle{definition}

\theoremstyle{remark}

\usepackage[textsize=tiny]{todonotes}
\setuptodonotes{inline}

\counterwithin{figure}{section}
\counterwithin{table}{section}

\title{The multiplicative Lipschitz bound as an inductive bias}

\author{\name Simon Kuang \email slku@ucdavis.edu \\
      \addr Department of Mechanical and Aerospace Engineering\\
      University of California, Davis
      \AND
      \name Yuezhu Xu \email xu1732@purdue.edu \\
      \addr Edwardson School of Industrial Engineering\\
      Purdue University
      \AND
      \name S. Sivaranjani \email sseetha@purdue.edu \\
      \addr Edwardson School of Industrial Engineering\\
      Purdue University
      \AND
      \name Xinfan Lin \email lxflin@ucdavis.edu \\
      \addr Department of Mechanical and Aerospace Engineering\\
      University of California, Davis}

\def\month{MM}  
\def\year{YYYY} 
\def\openreview{\url{https://openreview.net/forum?id=XXXX}} 

\begin{document}
\maketitle

\begin{abstract}
We study the problem of training a neural network to minimize a loss function given a Lipschitz constraint.
While the most obvious way to estimate the Lipschitz constant of a neural network is to multiply the Lipschitz constants of its layers,
this bound, which we term the \emph{trivial bound},
is often considered impractically loose.
We show that, on the contrary, it can be more effective than state-of-the-art methods (Sandwich, Almost-Orthogonal Layers, Convex Potential Layers, and SDP-based Lipschitz Layers) for Lipschitz-constrained learning.
Our method combines the trivial bound with a novel activation function (the quadrature nonlinearity) to achieve lower cross-entropy over most of the tested classification bounds.
For autoencoders constrained to be nonexpansive in a learned norm, it achieves the lowest pixel reconstruction error at every tested bottleneck width, improving on PCA by \(12\)--\(24\%\).
Our architecture, consisting of concatenated cosine and sine activations, is supported by a semi-rigorous theoretical analysis of Lipschitz bound looseness that may prove valuable in its own right.
\end{abstract}

{\em AI disclosure: we used generative AI to assist in the appendix and implementation, and we did not use generative AI to draft the main text.}

\section{Introduction}
\label{sec:intro}
We seek a function \(f:\mathbb R^{\mathsf d_\text{in}} \to \mathbb R^{\mathsf d_\text{out}}\) which minimizes a given loss functional, subject to a Lipschitz bound:
\begin{align}
  &\min_{f} \mathcal{L}[f] && \text{subject to} &\|f\|_{\text{Lip}} &\leq L^*
  \label{eq:infinite dimensional problem}
  \intertext{where $\mathcal{L}$ is the loss functional and $L^*$ is the maximum allowed Lipschitz constant. This infinite-dimensional problem is approximated by restricting \(f\) to a class of neural networks, parameterized by \(\theta\):}
  &\min_{\theta} \mathcal{L}[f_\theta] && \text{subject to} &\|f_\theta\|_{\text{Lip}} &\leq L^*,
  \label{eq:lipschitz-constrained-problem}
  \intertext{which remains an intractable problem, as there is in general no efficient algorithm to compute the Lipschitz constant of an arbitrary neural network \citep{virmaux_lipschitz_2018}.
  It is necessary, therefore, to approximate the Lipschitz constant using a tractable surrogate \(\hat L\):
  }
  &\min_{\theta} \mathcal{L}[f_\theta] && \text{subject to} &\hat L(\theta) &\leq L^*.
  \label{eq:surrogate-lipschitz-problem}
\end{align}
When does \eqref{eq:surrogate-lipschitz-problem} approximate \eqref{eq:infinite dimensional problem}?
\begin{enumerate}
  \item The set of functions \(\{f_\theta\}_{\theta}\) must be a good approximation to the set of all functions \(f\) with \(\|f\|_{\text{Lip}} \leq L^*\).
  This claim holds for large networks by universal approximation theory \citep{berner_modern_2022}, whose fine-grained applicability is outside the scope of this paper.
  \item The surrogate \(\hat L\) must be a valid \emph{upper} bound on \(\left\|f_\theta\right\|_{\text{Lip}}\). This ensures that solutions \(f_\theta\) to \eqref{eq:surrogate-lipschitz-problem} are \emph{feasible} under \eqref{eq:infinite dimensional problem}. 
  \item The surrogate \(\hat L\) should be as \emph{tight} as possible; i.e.~the feasible set of \eqref{eq:surrogate-lipschitz-problem} should approximate the feasible set of \eqref{eq:infinite dimensional problem} as closely as possible. This ensures that the solution is (nearly) \emph{optimal} under \eqref{eq:infinite dimensional problem}.
\end{enumerate}
At first glance, the most intuitive way to bound the Lipschitz constant of a neural network is to multiply the Lipschitz constants of its layers, resulting in the \textbf{trivial bound}.
The trivial bound is simple and computationally efficient.
However, the vast literature on the Lipschitz-bounded learning problem has eschewed the trivial bound in favor of significantly more difficult notions and computationally expensive (and condition-sensitive) algorithms.
We contextualize it by rendering explicit two assumptions that are usually left implicit.
\begin{intuition}
  \label{int:trivial-bound-cannot-be-tight}
  The trivial bound of a neural network with unconstrained weights is loose.\footnote{
  ``Simply computing the operator norm'' is ``conservative'' \citep[p.~9]{szegedy_intriguing_2014}.
The trivial bound ``may be subject to large errors'' \citep[p.~5]{virmaux_lipschitz_2018} and ``yields conservative results'' \citep[p.~3]{pauli_training_2021}.
``A product of the layer-wise Lipschitz constants [\ldots] is [\ldots] extremely loose and has many limitations'' \citep[p.~2]{chen_semialgebraic_2020}.
``The product of norms can be very large'' \citep[p.~401]{gouk_regularisation_2021}.
``Spectral-norm bounds can be quite loose'' \citep[p.~5]{wang_direct_2023}.
They are ``often too high to provide meaningful certificates'' \citep[p.~8]{bethune_deep_2024}.
}
\end{intuition}
\begin{intuition}
  \label{int:monotonic}
  Activation functions should be monotonic (tanh, ReLU) or nearly monotonic (GELU, Swish).
\end{intuition}

Intuition~\ref{int:trivial-bound-cannot-be-tight} can be sustained on principled grounds, which we explore in \S\ref{sec:weaknesses-trivial-bound}.
Empirically the trivial bound has been observed to be loose by a factor ranging from
drastic---\(10^{2}\) \citep[Figure 3]{zhang_efficient_2018}---to galactic---in excess of \(10^{20}\) \citep[Figure 2b]{fazlyab_efficient_2019}.

Intuition~\ref{int:monotonic} abides for social and traditional reasons.
On one hand, monotonic activations are bio-inspired by animal neurons, whose average propensity to fire can be modeled deterministically by a saturation function \citep[Figure 9]{braitenberg_vehicles_2004}.
While early theory on universal approximation relied on monotonicity \citep{cybenko_approximation_1989,funahashi_approximate_1989,hornik_multilayer_1989}, it is by no means necessary \citep{leshno_multilayer_1993,berner_modern_2022}.
On the other hand, slope restriction arose as a solution to the Lur'e problem in control theory, which seeks to classify nonlinear feedback laws that stabilize a linear recurrence \citep{fazlyab_efficient_2019}.

In the conventional research narrative, Intuition~\ref{int:trivial-bound-cannot-be-tight} is the problem, and Intuition~\ref{int:monotonic} is the solution.
It has focused on finding tighter Lipschitz bounds for networks with monotonic activations such as ReLU, mainly by using a linear matrix inequality to capture the fact that the input to a ReLU neuron is positively correlated with the output.
This idea originates as a method for using semidefinite programming (SDP) to certify the Lipschitz constant of a static network \citep{fazlyab_efficient_2019,xue_chordal_2022,xu_eclipse_2024,xu_eclipse-gen-local_2026}, i.e.~computing \(\theta\mapsto \hat L(\theta)\).
When it comes to training networks, there are various ways to impose the SDP relaxation of \(\hat L(\theta) \leq L^*\): by projection \citep{junnarkar_synthesizing_2024,xu_learning_2023} or through an implicit relationship \citep{revay_recurrent_2024}.
\citet{wang_direct_2023} develop an unconstrained parameterization of the SDP feasible set, and \citet{araujo_unified_2023} extend it to other architectures that share the ReLU activation.

Our work approaches this problem from a new direction.
We question Intuition~\ref{int:trivial-bound-cannot-be-tight} by sidestepping Intuition~\ref{int:monotonic}.
Instead we opt for a new activation function (the smooth \emph{quadrature} nonlinearity) and a novel application of a Lagrangian training method (the exponential method of multipliers),
motivated by the theses that looseness in the trivial bound can be mitigated in architecture and optimized away in training.
Empirically, we demonstrate the effectiveness of this combination for Lipschitz-constrained learning by achieving comparable or superior optimality in Problem~\eqref{eq:surrogate-lipschitz-problem} across different tasks \(\mathcal L\) and bounds \(L^*\).

\subsection{Related work}
\paragraph{Motivations for Lipschitz-constrained learning}
The global Lipschitz constant of a neural network is related, at least in principle, to the network's generalization (out of sample) and robustness (out of distribution) \citep{szegedy_intriguing_2014,
   bartlett_spectrally-normalized_2017,
   neyshabur_norm-based_2015,
   golowich_size-independent_2018,
   neyshabur_pac-bayesian_2018%
}.
These statistical concerns can be traced back to margin maximization in support vector machines; Problem~\ref{eq:infinite dimensional problem} can be motivated statistically as a nonlinear SVM.
Also, when learning dynamical systems, it is preferable to search for functions with Lipschitz constant \(\leq 1\), because they are stable when iterated \citep{xu_learning_2023,revay_recurrent_2024,manchester_neural_2026}.
Another application for Lipschitz-bounded learning is as the critic in a generative adversarial network, whose loss function attains the interpretation of Wasserstein distance \citep{arjovsky_wasserstein_2017}.

\paragraph{Lipschitz bounding (SDP)}
Semidefinite programming is a leading efficient method for Lipschitz-constrained learning \citep{manchester_neural_2026} and Lipschitz verification \citep{xu_eclipse-gen-local_2026}, and the primary comparison in our work, so we review it in detail.
The fundamental idea of SDP-based Lipschitz bounds is that activation functions satisfy a slope bound: \(\sigma'(x) \in [\alpha, \beta]\).
The working principle of monotonicity is as follows.
Let \(\mu > \lambda > 0\).
Consider the scalar function \(f(x) = \sigma(\mu x) - \sigma(\lambda x)\) across three cases:
\begin{itemize}
    \item (Perfect slope restriction)
    If \(\sigma' \in [1, 1]\) i.e.~\(\sigma(x) = x\), then obviously \(\left\|f\right\|_{\mathrm{Lip}} = \mu - \lambda\).
    \item
    (No slope restriction)
    If \(\sigma' \in [-1, 1]\), then \(\left\|f\right\|_{\mathrm{Lip}} \leq \mu + \lambda\).
    \item
    (Monotonicity)
    If \(\sigma' \in [0, 1]\), then the best we can do is
    \(\left\|f\right\|_{\mathrm{Lip}} \leq \mu\).
\end{itemize}
SDP-based methods exploit monotonicity to express the slope restriction as a quadratic inequality over increments of the activation:
\begin{align*}
  (\sigma(u) - \sigma(v) - \alpha(u - v))(\sigma(u) - \sigma(v) - \beta(u - v)) \leq 0,
\end{align*}
which is then used (by the lossy S-procedure) to imply the quadratic Lipschitz inequality \(\|f(u) - f(v)\|^2 \leq K^2\|u - v\|^2\).
This is the principle underlying LipSDP \citep{fazlyab_efficient_2019}, ECLipsE \citep{xu_eclipse_2024}, Sandwich \citep{wang_direct_2023}, and SLL \citep{araujo_unified_2023}.

Each method deploys semidefinite constraints differently.
On the verification side, 
LipSDP asks for all multipliers at once through one large block-tridiagonal semidefinite constraint;
ECLipsE factors the large constraint layer-by-layer into layer-wise metrics that telescope from the input to the penultimate layer (\cref{sec:eclipse-lipsdp-derivation}).
On the parameterization side, Sandwich uses a layer-wise decomposition of the LipSDP feasible set, and SLL extends the SDP principle to parameterize residual layers.


\paragraph{Lipschitz bounding (non-SDP methods)}
Apart from semidefinite programming, there are two major efforts toward Lipschitz estimation in the literature.
One seeks to compute the Lipschitz constant exactly using an exponential-time algorithm: mixed-integer linear programming \citep{jordan_exactly_2020}, mixed-integer quadratically constrained quadratic programming \citep{sbihi_miqcqp_2024}, branch-and-bound \citep{bhowmick_lipbab_2021,shi_neural_2025};
\citet{latorre_lipschitz_2020} pose the Lipschitz constant as an (NP-hard) quadratically constrained quadratic program before finding a semidefinite relaxation;
\citet{biertumpfel_lipschitz_2026} wrap an exponential-time algorithm around SDP in order to mitigate some of the issues raised in \Cref{sec:weaknesses-trivial-bound}.
A second effort computes \emph{local} Lipschitz bounds using bound propagation techniques to over-approximate the ranges of neurons
\citep{zhang_recurjac_2019,shi_efficiently_2022}.
But local Lipschitz bounds are qualitatively different from global bounds in that bound propagation becomes uninformative when the domain is large;
and for sufficiently small radii, local Lipschitz constants are polynomial-time computable to arbitrarily high accuracy.
For an up-to-date survey on the Lipschitz bounding problem, see \citet{xu_eclipse-gen-local_2026}.

\paragraph{Conventions and notation}
The vector of ones is \(\bm{1}\).
The Euclidean norm of \(x \in \mathbb R^d\) is \(\left\| x \right\|_2 = \sqrt{\sum_{i=1}^d |x_i|^2}\).
The spectral norm of a matrix \(A \in \mathbb R^{m \times n}\) is \(\left\| A \right\|_2 = \max_{\left\| x \right\|_2 = 1} \left\| A x \right\|_2 = \sqrt{\lambda_{\max}(A^\intercal A)}\).

A traditional feed-forward neural network of depth \(L\) is defined recursively by \(x^{0} = x\) and
\begin{align}
  x^{\ell} = \sigma^\ell\left(W^\ell x^{\ell-1} + b^\ell\right),
  \qquad \ell = 1,\dots,L,
  \label{eq:neural-network}
\end{align}
where \(\sigma^\ell\) acts elementwise and \(W^\ell\) and \(b^\ell\) are the weights and biases of layer \(\ell\). We write \(f(x) = x^{L}\).
Because induced matrix norms are submultiplicative,
the \textbf{trivial bound} on the Lipschitz constant of a feed-forward network is
\begin{align*}
  \left\|f\right\|_{\mathrm{Lip}} \leq \prod_{\ell=1}^L \left\|\sigma^\ell\right\|_{\mathrm{Lip}} \left\|W^\ell\right\|_2.
\end{align*}


If \(A\) is a matrix, \(\operatorname{Diag}(A)\) is the diagonal matrix with the same diagonal entries as \(A\) and zeros elsewhere.

The condition number of a (rectangular) matrix is the ratio of its largest singular value to its smallest \emph{nonzero} singular value.

\section{Motivation: mitigating weaknesses of the trivial bound}
\label{sec:weaknesses-trivial-bound}
This section unfolds a stylized narrative of four phenomena underlying the looseness of the trivial Lipschitz bound.
For each phenomenon, we construct a toy model that inflates the trivial Lipschitz bound and propose an architecture- and/or loss-level mitigation motivated by theoretical analysis.
We record the main takeaways as \textbf{Insights} 1--4, which shall be combined in \S\ref{sec:method} to yield a new algorithm for solving Lipschitz-constrained learning.
Proofs of the theorems in this section are given in \cref{sec:proofs}.

\subsection{Dead neurons}
\label{subsec:dead-neurons}
A ``dead neuron'' is a neuron whose preactivation lies in the saturation region of the nonlinearity for all inputs.
For example, a ReLU neuron whose preactivation is always negative is ``dead'' in that it has a zero Jacobian (ignores the network's input) and zero gradient (never updated during unregularized training).
An example of a dead neuron is the mapping \(x \mapsto z\) given by
\begin{align*}
  y &= \operatorname{ReLU}(1000 x), &
  z &= \operatorname{ReLU}(-1000 y).
\end{align*}
Here, the trivial bound on the Lipschitz constant from \(x\) to \(z\) equals \(10^6\), but the true Lipschitz constant is 0, because the argument of the second layer's ReLU is nonpositive for every \(x\).

The phenomenon is not limited to tabletop examples.
Consider a scaled ReLU neuron \(z^L\) at layer \(L\) and a predecessor neuron \(z^{L-1}\):
\begin{align}
  \label{eq:dead-neuron-example}
  z^L &= \alpha \operatorname{ReLU}(\overline z - \underline z - \operatorname{ReLU}(z^{L-1} - \underline{z})),
\end{align}
where \(\alpha \gg 1\) is a large weight such as 1 million.
If the range of \(z^{L-1}\) has a nonempty intersection with \([\underline z, \overline z]\), then the Lipschitz constant of the network is 1 million times greater than if it were not.

The remedy we propose is to replace elementwise nonlinearities like ReLU, tanh, etc. with a novel nonlinearity:

\begin{insight}[Isometric Jacobians]
\label{ins:isometric-jacobians}
The vector-valued function \(\mathbb R^n \to \mathbb R^{2n}\), \(x \mapsto (\cos x, \sin x)\) is everywhere a local isometry, i.e.~all of the singular values of the Jacobian are equal to 1.
\end{insight}

\subsection{Biases}
\label{subsec:biases}

Observe that \cref{eq:neural-network} is not the only way to write the composition
\begin{gather*}
  \operatorname{activation}_{\sigma^\ell} \circ \operatorname{shift}_{b^\ell} \circ \operatorname{linear}_{W^\ell}
  \circ \operatorname{activation}_{\sigma^{\ell-1}} \circ \operatorname{shift}_{b^{\ell-1}} \circ \operatorname{linear}_{W^{\ell-1}}
  \circ
  \dots
\end{gather*}
where \(\operatorname{activation}\) denotes the elementwise activation, \(\operatorname{shift}\) denotes translation by the bias, and \(\operatorname{linear}\) denotes matrix multiplication by the weights.
Taking circular permutations, we see that there are three ways to associate these operations:
\begin{align*}
  &\text{\textbf activation}\circ \text{\textbf shift}\circ \text{\textbf linear},
  \tag{ASL}
  \\
  &\text{\textbf shift}\circ \text{\textbf linear}\circ \text{\textbf activation},
  &\text{and}
  \tag{SLA}
  \\
  &\text{\textbf linear}\circ \text{\textbf activation}\circ \text{\textbf shift}.
  \tag{LAS}
\end{align*}
The trivial bound is tight for compositionally shallow slices of a neural network:
\begin{theorem}
  \label{thm:trivial-bound-tight-asl-sla-las}
  The trivial Lipschitz bound in the 2-norm is
  tight for ASL units with generic \(W \in \mathbb R^{d \times d}\) and \(b \in \mathbb R^d\),
  and for SLA and LAS units with arbitrary \(W \in \mathbb R^{d \times d}\) and \(b \in \mathbb R^d\).
\end{theorem}
However, none of these structures is expressive enough to be a universally approximating class.
For that we need the quadruple composition
\begin{gather*}
  \text{\textbf linear}\circ \text{\textbf activation}\circ \text{\textbf shift}\circ \text{\textbf linear}.
  \tag{LASL}
\end{gather*}
Unfortunately, the trivial Lipschitz bound in the 2-norm can be seen to be arbitrarily loose for LASL units \(x \mapsto W^1 \sigma(W^0x + b)\), as we see from three different explicit constructions.

\begin{theorem}
  \label{thm:lasl-biases-relu-tanh}
  For both the ReLU and tanh activation functions, there exist LASL units \(\mathbb R \to \mathbb R^d \to \mathbb R\)
  on which the trivial 2-Lipschitz bound is \(\Omega(d)\),
  but the true 2-Lipschitz constant ranges from \(0\) to \(\Omega(d)\) by varying only the biases.
\end{theorem}

\begin{theorem}
  \label{thm:lasl-biases-sin}
  For the sine activation function, there exist LASL units \(\mathbb R \to \mathbb R^{2d} \to \mathbb R\) on which the trivial Lipschitz bound is \(\Omega(d)\), but the true Lipschitz constant is 0.
\end{theorem}

Thus a fully differential method such as SDP can be infinitely loose.

A note on the proof: in all of these constructions, the reason why biases make all the difference is that they toggle whether the neurons in a layer interfere constructively or destructively, i.e.~whether the neurons in the layer attain their maximum slopes at the same location (constructively) or far apart and in opposite signs (destructively).
But because SDP methods ignore biases, they are forced to over-approximate by imputing to the biases their worst-case values.
This leads to 
up to \(\Omega(d)\) conservatism per layer on ReLU and tanh networks and infinite conservatism on sinusoidal layers, and hence to exponential looseness in deeper networks.

What can be done?
It is hard to calculate whether biases are in or out of phase.
Removing the \textbf shift parameter from LA\textbf SL units would eliminate this source of ambiguity outright, but at the expense of universal approximation.
What we need is a way to capture the expressivity of multilayer perceptrons, but force the network to reveal its phase information in a linear-algebraic way.

Let us reconsider the situation of \cref{thm:lasl-biases-sin}, where two sinusoid neurons are being combined with weights \(w_i\) and biases \(b_i\):
\begin{align}
  y &= w_1 \sin (x + b_1) + w_2 \sin(x + b_2)
  \\
  \intertext{Using trigonometric identities, this is equivalent to}
  y 
  &= \underbrace{\del{w_1 \sin b_1 + w_2 \sin b_2}}_{A} \cos x + \underbrace{\del{w_1 \cos b_1 + w_2 \cos b_2}}_{B} \sin x.
  \label{eq:sinusoid-combination}
\end{align}
In the negative result \cref{thm:lasl-biases-sin}, the trivial bound was infinitely loose for \(w_1 = w_2\), \(b_1 + \pi = b_2\).
But in such a case \(A = B =0\), and the Lipschitz constant is evidently zero.
Equation \eqref{eq:sinusoid-combination} inspires the hidden layer
\(y = A \cos (x) + B \sin(x)\) where \(\cos\) and \(\sin\) are elementwise.

Without loss of generality, the LASL building block becomes
\begin{gather}
  x \mapsto A \cos(Wx) + B \sin(Wx),
  \label{eq:lasl-unit-simplified}
\end{gather}
which is a multivariate Fourier partial sum and attains universal approximation without explicit biases.

\begin{insight}[Absorb biases into weights]
  \label{ins:absorb-biases}
  By trigonometric identities, the nonlinearity \(x \mapsto (\cos x, \sin x)\) allows biases (which are hard to verify) to be absorbed into weights (which are easy to verify).
\end{insight}

\subsection{Activation rescaling}
\label{subsec:ill-conditioned-weight-matrices}
A network with ill-conditioned weight matrices can be less Lipschitz than it appears.
This occurs when a neural network contains ill-conditioned subnetworks that cancel out.

\paragraph{Reparameterization of homogeneous activation functions}
The ReLU activation function is positively 1-homogeneous: for all \(\alpha \in (0, \infty)\), \(\operatorname{ReLU}(\alpha x) = \alpha \operatorname{ReLU}(x)\).
This means that if \(\Lambda\) is a positive diagonal matrix, then any two consecutive layers can be reparameterized 
\((W^0, b^0, W^1) \mapsto (\Lambda^{-1} W^0, \Lambda^{-1} b^0, W^1\Lambda)\), while the composite function
\begin{gather}
  W^1 \operatorname{ReLU}(W^0 x + b^0) + b^1 
  =  W^1 \Lambda \operatorname{ReLU}(\Lambda^{-1} W^0  x + \Lambda^{-1} b^0) + b^1
  \label{eq:gauge}
\end{gather}
remains unchanged.\footnote{This fact has been applied in another context to improve the training of ReLU networks
 \citep{
   neyshabur_path-sgd_2015,
   neyshabur_data-dependent_2016}.}
Given generic weight matrices \(W^0, W^1\), reparameterization by
\begin{align*}
  \Lambda &= \diag(\lambda, 1, 1, 1, \ldots), & \lambda \to \infty
\end{align*}
results in a trivial Lipschitz bound of \(\Theta(\lambda)\) times the original Lipschitz constant.
We record this observation as:
\begin{theorem}
  \label{thm:relu-network-trivial-bound-arbitrarily-large}
  Given any ReLU network, there exists a closed-form, mathematically equivalent network whose trivial 2-norm Lipschitz bound is arbitrarily large.
\end{theorem}
While this form of ill-conditioning also evades analysis methods such as LipSDP-Layer \citep{fazlyab_efficient_2019}, which do not know the homogeneity of ReLU,
LipSDP-Neuron is equivariant under the diagonal rescaling gauge group.
This form of ill-conditioning is one of possibly many qualitative features of a neural network that SDPs are able to capture without understanding true reachability.
In this case, SDPs' ellipsoidal over-approximation to reachability \eqref{eq:eclipse-telescope-hidden}, while too weak to solve the NP-hard version of reachability, is sufficient to absorb this form of ill-conditioning.

But is it necessary? Reproducing a 3-layer, width-100 ReLU network, trained on MNIST \citep{xu_learning_2023},
  we get a trivial Lipschitz bound of 26.78.
  ECLipsE obtains a Lipschitz bound of 17.88 in 04:46 (mm:ss); LipSDP \citep{fazlyab_efficient_2019} is expected to take several times longer.
  On the other hand, we obtain a Lipschitz bound of 21.94 in 00:31 by using a standard quasi-Newton optimizer to minimize the trivial bound by searching for a good diagonal reparameterization \eqref{eq:gauge}.
  This tabletop experiment shows that more than half of SDP's tightness can be attributed to optimal diagonal scaling, which can be performed in about one-tenth of the time.
  In other scenarios, such as when the network is trained using an adaptive optimizer, the accuracy gap between ECLipsE/LipSDP and gauge-adjusted trivial bounds all but vanishes.
  The methodology is described in \cref{subsec:mnist-gauge-methodology}.

\paragraph{Analysis of non-homogeneous activation functions}
Moreover, the same degree of freedom can also be exploited differentially to improve the trivial bound for a non-homogeneous activation function, which without loss of generality we take to be 1-Lipschitz.
The network Jacobian can be manipulated as:
\begin{align*}
  \dpd{}{x}f(x)
  &= W^L D^{L-1} W^{L-1} D^{L-2} \cdots W^2 D^1 W^1, \quad \text{where } D^\ell = \operatorname{diag}(\sigma'(z^\ell))
  \\
  &= W^L \Lambda_{L-1} \Lambda_{L-1}^{-1} D^{L-1} \Lambda_{L-1}
   \Lambda_{L-1}^{-1}
  W^{L-1} 
  \Lambda_{L-2} \Lambda_{L-2}^{-1} D^{L-2} \Lambda_{L-2} \Lambda_{L-2}^{-1}
  \cdots W^2 D^1 W^1
\end{align*}
by using positive diagonal multipliers \(\Lambda_\ell\) to enact \(D \mapsto \Lambda \Lambda^{-1} D \Lambda \Lambda^{-1}\) at each layer.
Re-associating this matrix product, commuting diagonal matrices
\(\Lambda^{-1} D \Lambda = D\),
taking operator norms to derive a decoupled ``trivial'' bound, and using the hypothesis \(|\sigma'(x)| \leq 1\),
we get
\begin{align}
  \left\|\dpd{}{x}f(x)\right\|
  &\leq 
  \inf_{\{\Lambda_\ell\}_{\ell=1}^{L-1} }\left\| W^L \Lambda_{L-1}\right\| 
   \left\|\Lambda_{L-1}^{-1}
  W^{L-1} 
  \Lambda_{L-2}\right\|
  \cdots \left\|\Lambda_1^{-1} W^1\right\|.
  \label{eq:diagonal-rescaling-bound}
\end{align}
The \(\{\Lambda_\ell\}_{\ell=1}^{L-1}\) decision variables represent degrees of freedom to re-interpret the weights of \emph{any} MLP to improve the trivial bound while keeping the network's output unchanged.


We propose to absorb them into a parametric nonlinearity, the \textbf{quadrature} nonlinearity
\(\mathbb R^n \times (\mathbb R_+)^n \to \mathbb R^{2n}\),
\(x,  \Lambda \mapsto (\Lambda \cos(\Lambda^{-1} x), \Lambda \sin (\Lambda^{-1}x))\):

\begin{insight}[Gauge invariance]
  \label{ins:gauge-invariance}
  Insofar as \(\Lambda\) affects the function but not the bound, training it allows a neural network's weight matrices to solve \eqref{eq:diagonal-rescaling-bound} by continuous optimization.
\end{insight}

\subsection{Ill-conditioning (matrix factorization model)}
To situate the effect of ill-conditioning in a ``spherical cow'' setting, let us consider the problem of learning an invertible linear function as a deep neural network with \emph{linear} activation function \(\sigma(x) = x\).
While the use of linear activations is impractical for expressivity (as linear layers are closed under composition), linear models can be enlightening to understand overparameterized deep networks
\citep{zhang_depth_2020,nam_position_2025,oliveira_convergence_2026,shewchuk_geometry_2024,shewchuk_hierarchy_2026}.

Without loss of generality, the objective is to learn the identity function using the mapping
\begin{equation}
  \label{eq:linear-composition}
  x \mapsto W_L W_{L-1} \cdots W_1 x
\end{equation}
where the parameters \(W_1, W_2, \ldots, W_L\) are square matrices.
  
\begin{theorem}
  \label{thm:condition-number-tightness-product-identity}
  Consider a deep network \eqref{eq:linear-composition} with linear activations and parameter matrices 
  constrained by 
  \(W_L W_{L-1} \cdots W_1 = I\).
  Let \(K = \left\|W_1\right\|_2 \left\|W_2\right\|_2 \cdots \left\|W_L\right\|_2\) be the trivial Lipschitz bound.
  Then the following are equivalent:
  \begin{enumerate}
    \item \label{item:K-equals-1} \(K = 1\), i.e.~there is no conservatism in the trivial bound.
    \item \label{item:minimize-K} \(W_1, W_2, \ldots, W_L\) minimize \(K\) subject to \(W_L W_{L-1} \cdots W_1 = I\).
    \item \label{item:condition-number-1} The condition number of each parameter matrix is 1.
  \end{enumerate}
\end{theorem}

In this toy example, minimizing the trivial Lipschitz bound over slack degrees of freedom on a level set of the loss function recovers the ground truth Lipschitz constant.
Whereas it is already established that overparameterization is necessary for Lipschitz learning \citep{bubeck_universal_2023},
here we offer an exploitable mechanism:
because neural networks are in general overparameterized, there may be a descent direction which lowers the trivial bound while leaving the network essentially unchanged as a mathematical function.\footnote{This observation explains a side finding of \citet[Fig.~4(c)]{leino_globally-robust_2021}, in which a trivial bound becomes tight in response to optimization pressure.}
In slogan form,
\begin{insight}[Penalize the trivial bound]
  \label{ins:penalized-tightness}
  In overparameterized networks, the trivial bound \emph{becomes} tight when penalized.
\end{insight}

\section{Method: quadrature activation and Lagrangian training}
\label{sec:method}
We assemble Insights~\ref{ins:isometric-jacobians}--\ref{ins:penalized-tightness} into a single algorithm comprising a novel architecture---quadrature---and a novel training method---the exponential method of multipliers.
A \textbf{quadrature network}\footnote{In electrical engineering, the term ``quadrature'' refers to the use of both cosine and sine basis functions.} of width \(N\) and depth \(L\) is defined by
\begin{subequations}
\label{eq:quadrature-network}
\begin{align}
  f_\theta(x) &= y^L + b,
  \label{eq:quadrature-output}
  \\
  y^\ell &= A^{\ell} \Lambda^\ell \cos((\Lambda^{\ell})^{-1} y^{\ell-1}) + B^\ell \Lambda^\ell \sin((\Lambda^{\ell})^{-1} y^{\ell-1}) \in \mathbb R^N, & \ell &\in [1\ldots L],
  \label{eq:quadrature-hidden}
  \\
  y^0 &= Wx,
  \label{eq:quadrature-input}
\end{align}
\end{subequations}
where \(\theta = (W, A^1, B^1, \Lambda^1, \dots, A^L, B^L, \Lambda^L, b)\) and \(\Lambda^1,\ldots,\Lambda^L\) are positive diagonal matrices.
We define its trivial Lipschitz bound as:
\begin{align}
  \hat K(\theta) &= \left\|W\right\| \prod_{\ell = 1}^L \left\|\begin{pmatrix}
    A^\ell & B^\ell
  \end{pmatrix}
  \right\|.
  \label{eq:trivial-lipschitz-bound}
\end{align}

\begin{theorem}
  \label{thm:trivial-bound-valid}
  The function \(\hat K(\theta)\) defined in \eqref{eq:trivial-lipschitz-bound} is a valid Lipschitz bound for the neural network \eqref{eq:quadrature-network}.
\end{theorem}
\begin{proof}
  Recall that the Lipschitz constant is submultiplicative (for all composable functions \(f\) and \(g\), we have \(\left\|f \circ g\right\|_{\mathrm{Lip}} \leq \left\|f\right\|_{\mathrm{Lip}} \left\|g\right\|_{\mathrm{Lip}}\))
  and that the Lipschitz constant of a linear function is precisely its operator norm.
  The only nontrivial step is to prove that for any positive diagonal \(\Lambda\),
  the map
  \begin{align}
    h(v) &=
    \begin{pmatrix}
      \Lambda \cos \Lambda^{-1} v \\
      \Lambda \sin \Lambda^{-1} v
    \end{pmatrix}
    \intertext{is 1-Lipschitz; this follows \emph{a fortiori} from the fact that it is a local isometry. The Jacobian}
    Dh(v)
    &= 
    \begin{pmatrix}
      -\Lambda \diag \del{\Lambda^{-1} \sin \Lambda^{-1} v}
      \\
      \phantom{-}\Lambda \diag \del{\Lambda^{-1} \cos \Lambda^{-1} v}
    \end{pmatrix}
    \\
    &= 
    \begin{pmatrix}
      -\diag \del{\sin \Lambda^{-1} v}
      \\
      \phantom{-} \diag \del{\cos \Lambda^{-1} v}
    \end{pmatrix}
    \intertext{satisfies}
    Dh(v)^\intercal Dh(v)
    &= \diag\del{\sin^2 \Lambda^{-1} v} + \diag\del{\cos^2 \Lambda^{-1} v}
    \\
    &= I
  \end{align}
  by the Pythagorean Identity.
\end{proof}

In the \textbf{exponential method of multipliers}\footnote{
  A method for inequality-constrained programming, first proposed by \citet{kort_new_1972} and proved convergent for convex objectives by \citet{tseng_convergence_1993}, that is now seldom preferred over modern interior-point methods.
} (EMM),
a black-box optimizer approximately minimizes the penalized objective:
\begin{subequations}
\label{eq:emm}
\begin{align}
\theta &\gets \operatorname{argmin}_{\theta} \mathcal{L}[f_\theta] + \mu \hat K(\theta)
\intertext{for some Lagrange multiplier \(\mu > 0\), periodically updated by}
\mu &\gets \frac{\hat K(\theta)}{L^*} \mu.
\end{align}
\end{subequations}

\section{Results}
\label{sec:results}
The results in this section compare our method to other relaxations of Problem~\ref{eq:infinite dimensional problem} to Problem~\ref{eq:surrogate-lipschitz-problem}.
The natural criterion is which method achieves the lowest loss.
Both experiments include four SDP-certified ReLU architectures that satisfy \(\hat L(\theta) \leq L^*\) by construction.
\begin{description}
  \item[Sandwich \citep{wang_direct_2023}] A direct parameterization of the feasible set of LipSDP \citep{fazlyab_efficient_2019}.
  Hidden layers take the form \(h \mapsto \sqrt 2 A^\top \Psi \operatorname{ReLU}\del{\sqrt 2 \Psi^{-1} B h + b}\), where \(\Psi\) is a positive diagonal matrix and \(A A^\top + B B^\top = I\) is enforced by a Cayley transform.
  \item[AOL \citep{prach_almost-orthogonal_2022}] Almost-Orthogonal Layers rescale the columns of each weight matrix, \(h \mapsto \operatorname{ReLU}\del{W T^{-1/2} h + b}\) with \(T = \diag\del{\sum_j \left|W^\top W\right|_{ij}}\), so that \(W^\top W \preceq T\) by the Gershgorin circle theorem and each layer is 1-Lipschitz.
  \item[CPL \citep{meunier_dynamical_2022}] Convex Potential Layers are residual blocks \(h \mapsto h - \frac{2}{\left\|W\right\|_2^2} W \operatorname{ReLU}\del{W^\top h + b}\), arising as a gradient step on a convex potential, which is 1-Lipschitz because \(\operatorname{ReLU}\) is slope-restricted to \([0, 1]\).
  \item[SLL \citep{araujo_unified_2023}] SDP-based Lipschitz Layers generalize CPL to \(h \mapsto h - 2 W T^{-1} \operatorname{ReLU}\del{W^\top h + b}\) for any diagonal \(T \succeq W^\top W\), chosen analytically as \(T_{ii} = \sum_j \left|W^\top W\right|_{ij}\, q_j / q_i\) with trainable \(q > 0\).
  CPL is the special case \(T = \left\|W\right\|_2^2 I\), and \(q = \bm 1\) recovers the AOL choice of \(T\).
\end{description}
Our method, \textbf{Quadrature-EMM}, is ablated against \textbf{ReLU-EMM}, which applies the training law
\eqref{eq:emm} to a ReLU MLP.
The autoencoder comparison additionally includes two other baselines which employ a locally isometric activation function:
\begin{description}
  \item[FullSort \citep{anil_sorting_2019}] combines normalized linear maps with the layer-wise activation that sorts the coordinates in a vector.
  \item[MaxMin \citep{anil_sorting_2019}] combines normalized linear maps with the activation that sorts pairs of coordinates in a vector.
\end{description}
(In the classification example, we omit these baselines due to their excessive runtime.)

\subsection{Minimum cross-entropy loss subject to Lipschitz bound}
\label{subsec:results-classification}
We train a classifier network using \(\mathcal L\), the cross-entropy loss, on the MNIST training set, and we sweep \(L^*\) across ten log-spaced values from \(10\) to \(100\).
(Implementation details are given in \cref{sec:appendix-classification}.)
Table~\ref{tab:sweep-loss} reports the training loss as a geometric exponential
moving average of the minibatch loss
with a time constant of approximately \(1{,}000\) included minibatches; \cref{fig:sweep-summary} plots the same data against the certified bound.
For the EMM methods, even though the training trajectory straddles the feasibility boundary, in order to assess our methods conservatively, we compute the moving average only of those steps at which the network was feasible (\(\hat K \leq L^*\)).

Quadrature-EMM attains the lowest training loss at seven of the ten bounds, by a factor of \(2.5\) relative to Sandwich at \(L^* = 100\), while CPL is lowest at \(L^* \leq 16.7\).
(Even though the present work does not have a statistical emphasis, we note that test accuracy is between \(98.05\%\) and \(98.79\%\) for every method and every bound.)

\begin{table}[h]
\centering
\caption{MNIST classification: geometric moving average of training cross-entropy at each target bound \(L^*\), restricted to feasible iterates for EMM. Lowest loss per row in bold.}
\label{tab:sweep-loss}
\small
\begin{tabular}{@{}lrrrrrr@{}}
\toprule
\(L^*\) & Sandwich & AOL & CPL & SLL & ReLU-EMM & Quadrature-EMM \\
\midrule
\(10\) & \(4.40 \times 10^{-4}\) & \(1.66 \times 10^{-2}\) & \(\bm{3.05 \times 10^{-4}}\) & \(3.95 \times 10^{-4}\) & \(3.30 \times 10^{-3}\) & \(5.48 \times 10^{-4}\) \\
\(12.9\) & \(1.36 \times 10^{-4}\) & \(9.52 \times 10^{-3}\) & \(\bm{9.82 \times 10^{-5}}\) & \(1.31 \times 10^{-4}\) & \(1.10 \times 10^{-3}\) & \(1.48 \times 10^{-4}\) \\
\(16.7\) & \(4.17 \times 10^{-5}\) & \(4.60 \times 10^{-3}\) & \(\bm{3.16 \times 10^{-5}}\) & \(4.80 \times 10^{-5}\) & \(3.22 \times 10^{-4}\) & \(3.66 \times 10^{-5}\) \\
\(21.5\) & \(9.85 \times 10^{-6}\) & \(2.10 \times 10^{-3}\) & \(1.25 \times 10^{-5}\) & \(2.64 \times 10^{-5}\) & \(8.39 \times 10^{-5}\) & \(\bm{8.97 \times 10^{-6}}\) \\
\(27.8\) & \(4.86 \times 10^{-6}\) & \(7.96 \times 10^{-4}\) & \(1.04 \times 10^{-5}\) & \(1.24 \times 10^{-5}\) & \(2.04 \times 10^{-5}\) & \(\bm{1.70 \times 10^{-6}}\) \\
\(35.9\) & \(1.94 \times 10^{-6}\) & \(2.49 \times 10^{-4}\) & \(4.57 \times 10^{-6}\) & \(1.04 \times 10^{-5}\) & \(4.70 \times 10^{-6}\) & \(\bm{4.09 \times 10^{-7}}\) \\
\(46.4\) & \(8.02 \times 10^{-7}\) & \(8.28 \times 10^{-5}\) & \(3.60 \times 10^{-6}\) & \(6.74 \times 10^{-6}\) & \(9.28 \times 10^{-7}\) & \(\bm{2.07 \times 10^{-7}}\) \\
\(59.9\) & \(1.94 \times 10^{-7}\) & \(2.45 \times 10^{-5}\) & \(1.33 \times 10^{-6}\) & \(7.04 \times 10^{-6}\) & \(1.81 \times 10^{-7}\) & \(\bm{3.21 \times 10^{-8}}\) \\
\(77.4\) & \(4.23 \times 10^{-8}\) & \(1.22 \times 10^{-5}\) & \(1.33 \times 10^{-6}\) & \(4.83 \times 10^{-6}\) & \(4.05 \times 10^{-8}\) & \(\bm{1.54 \times 10^{-8}}\) \\
\(100\) & \(2.09 \times 10^{-8}\) & \(5.51 \times 10^{-6}\) & \(2.93 \times 10^{-7}\) & \(2.35 \times 10^{-6}\) & \(1.52 \times 10^{-8}\) & \(\bm{8.40 \times 10^{-9}}\) \\
\bottomrule
\end{tabular}
\end{table}

\subsection{Autoencoders}
\label{subsec:results-autoencoder}
We train an autoencoder using the reconstruction loss
\begin{align}
  \mathcal L[f, U]
  &= \frac{1}{n} \sum_{i=1}^n \left\|x_i - U f (U^{-1} x_i)\right\|^2,
  \label{eq:autoencoder-loss}
\end{align}
where \(\{x_i\}_{i=1}^n\) is the training set,
\(U\) is an upper-triangular matrix with a positive diagonal,
 and \(f\) is a neural network with bottleneck dimension \(d\).
We restrict \(f\) to Lipschitz constant at most 1,
so that the autoencoder map \(x \mapsto U f(U^{-1} x)\) is non-expansive in the metric \(\left\|x\right\|_U = \left\|U^{-1}x\right\|\).
(The Cholesky decomposition guarantees a one-to-one correspondence between quadratic metrics and positive-diagonal upper-triangular matrices.)

This choice is motivated by a strong interpretation of the manifold hypothesis: given that the data lives on a low-dimensional submanifold of the ambient space, an autoencoder that reconstructs the manifold exactly must have Lipschitz constant no less than 1, because its differential must restrict to the identity on the tangent bundle of the data manifold.
The number 1 is also a meaningful upper bound, as it is the Lipschitz constant of orthogonal projection onto principal components.

We sweep the bottleneck dimension \(d \in \{8, 16, 32, 64, 128\}\), and additionally baseline against principal component analysis (PCA), which is the closed-form solution to \eqref{eq:autoencoder-loss} among all affine \(f\).
(Implementation details are given in \cref{sec:appendix-autoencoder}.)

For both EMM methods we report the last iterate with \(\hat K \leq 1\).
Quadrature-EMM attains the lowest training and test MSE at every bottleneck width, with training MSE \(12\)--\(24\%\) below PCA (Table~\ref{tab:autoencoder-mse}).
Notably, all of the baseline methods for Lipschitz-constrained nonlinear learning perform on par with, or worse than PCA at this task.

\begin{table}[h]
\centering
\caption{MNIST autoencoders: training pixel MSE in units of \(10^{-3}\). Each neural model is constrained to be 1-Lipschitz in its own learned norm; PCA uses the Euclidean norm. Lowest MSE per column in bold.}
\label{tab:autoencoder-mse}
\small
\begin{tabular}{@{}lrrrrr@{}}
\toprule
Method & \(d=8\) & \(d=16\) & \(d=32\) & \(d=64\) & \(d=128\) \\
\midrule
PCA & \(37.87\) & \(27.29\) & \(17.24\) & \(9.29\) & \(4.29\) \\
Sandwich & \(42.01\) & \(30.49\) & \(19.10\) & \(12.02\) & \(7.77\) \\
AOL & \(38.85\) & \(28.61\) & \(20.18\) & \(10.11\) & \(5.11\) \\
CPL & \(37.87\) & \(27.30\) & \(17.25\) & \(9.29\) & \(4.27\) \\
SLL & \(37.87\) & \(27.30\) & \(17.25\) & \(9.29\) & \(4.27\) \\
FullSort & \(37.87\) & \(27.30\) & \(17.26\) & \(9.39\) & \(6.59\) \\
MaxMin & \(37.87\) & \(27.29\) & \(17.24\) & \(9.28\) & \(4.27\) \\
ReLU-EMM & \(37.87\) & \(28.29\) & \(19.09\) & \(10.61\) & \(6.02\) \\
Quadrature-EMM & \(\bm{33.37}\) & \(\bm{22.66}\) & \(\bm{13.28}\) & \(\bm{7.03}\) & \(\bm{3.38}\) \\
\bottomrule
\end{tabular}
\end{table}

\section{Conclusion}
Simple is best.
While other works have addressed the looseness of the trivial bound by adding more sophisticated analysis to activation functions such as ReLU and tanh, our work instead applies a theoretically-motivated design process to derive the \textbf{quadrature} nonlinearity, which promotes tightness by design.
We find that the simplicity of the trivial bound is rewarded with superior performance in Lipschitz-constrained optimization problems.

Our results imply that certifiable safety and performance are more within reach in applied machine learning.
The Lipschitz constant of a neural network is less of a mathematical abstraction and more of an explicit design decision.


\clearpage

\section*{Reproducibility Statement}
The experiment implementations are \texttt{mnist\_lipschitz\_sweep.py} and \texttt{mnist\_autoencoder\_lipschitz\_comparison.py} in \texttt{rectangular-sinusoid-networks}.
The manuscript uses the classification configuration, CSV/JSON results, and summary in \texttt{assets/mnist\_lipschitz\_sweep/mnist/} and the autoencoder configuration, CSV/JSON results, and figures in \texttt{assets/mnist\_autoencoder\_lipschitz\_comparison/}.
The autoencoder artifacts specify seed \(42\), shuffling seed \(0\), and \texttt{learn\_norm=true}.
The scripts also save configurations, training histories, and model arrays in results shelves.

\bibliographystyle{tmlr}
\bibliography{lipschitz}

@book{braitenberg_vehicles_2004,
	address = {Cambridge, Mass.},
	edition = {9. print},
	series = {Bradford book psychology},
	title = {Vehicles: experiments in synthetic psychology},
	isbn = {978-0-262-52112-3 978-0-262-02208-8},
	shorttitle = {Vehicles},
	language = {eng},
	publisher = {MIT Press},
	author = {Braitenberg, Valentino},
	year = {2004},
}

@misc{junnarkar_synthesizing_2024,
	title = {Synthesizing {Neural} {Network} {Controllers} with {Closed}-{Loop} {Dissipativity} {Guarantees}},
	url = {http://arxiv.org/abs/2404.07373},
	doi = {10.48550/arXiv.2404.07373},
	urldate = {2024-08-26},
	publisher = {arXiv},
	author = {Junnarkar, Neelay and Arcak, Murat and Seiler, Peter},
	month = apr,
	year = {2024},
	note = {arXiv:2404.07373 [cs, eess]},
}

@article{revay_recurrent_2024,
	title = {Recurrent {Equilibrium} {Networks}: {Flexible} {Dynamic} {Models} {With} {Guaranteed} {Stability} and {Robustness}},
	volume = {69},
	issn = {1558-2523},
	shorttitle = {Recurrent {Equilibrium} {Networks}},
	url = {https://ieeexplore.ieee.org/document/10179161},
	doi = {10.1109/TAC.2023.3294101},
	number = {5},
	urldate = {2025-05-09},
	journal = {IEEE Transactions on Automatic Control},
	author = {Revay, Max and Wang, Ruigang and Manchester, Ian R.},
	month = may,
	year = {2024},
	pages = {2855--2870},
}

@article{virtanen_scipy_2020,
	title = {{SciPy} 1.0: {Fundamental} {Algorithms} for {Scientific} {Computing} in {Python}},
	volume = {17},
	url = {https://doi.org/10.1038/s41592-019-0686-2},
	doi = {10.1038/s41592-019-0686-2},
	urldate = {2025-06-25},
	journal = {Nature Methods},
	author = {Virtanen, Pauli and Gommers, Ralf and Oliphant, Travis E. and Haberland, Matt and Reddy, Tyler and Cournapeau, David and Burovski, Evgeni and Peterson, Pearu and Weckesser, Warren and Bright, Jonathan and van der Walt, Stéfan J. and Brett, Matthew and Wilson, Joshua and Millman, K. Jarrod and Mayorov, Nikolay and Nelson, Andrew R. J. and Jones, Eric and Kern, Robert and Larson, Eric and Carey, C J and Polat, Ilhan and Feng, Yu and Moore, Eric W. and VanderPlas, Jake and Laxalde, Denis and Perktold, Josef and Cimrman, Robert and Henriksen, Ian and Quintero, E. A. and Harris, Charles R. and Archibald, Anne M. and Ribeiro, Antônio H. and Pedregosa, Fabian and van Mulbregt, Paul and {SciPy 1.0 Contributors}},
	year = {2020},
	pages = {261--272},
}

@misc{deepmind_deepmind_2020,
	title = {The {DeepMind} {JAX} {Ecosystem}},
	url = {http://github.com/google-deepmind},
	author = {{DeepMind} and Babuschkin, Igor and Baumli, Kate and Bell, Alison and Bhupatiraju, Surya and Bruce, Jake and Buchlovsky, Peter and Budden, David and Cai, Trevor and Clark, Aidan and Danihelka, Ivo and Dedieu, Antoine and Fantacci, Claudio and Godwin, Jonathan and Jones, Chris and Hemsley, Ross and Hennigan, Tom and Hessel, Matteo and Hou, Shaobo and Kapturowski, Steven and Keck, Thomas and Kemaev, Iurii and King, Michael and Kunesch, Markus and Martens, Lena and Merzic, Hamza and Mikulik, Vladimir and Norman, Tamara and Papamakarios, George and Quan, John and Ring, Roman and Ruiz, Francisco and Sanchez, Alvaro and Sartran, Laurent and Schneider, Rosalia and Sezener, Eren and Spencer, Stephen and Srinivasan, Srivatsan and Stanojević, Miloš and Stokowiec, Wojciech and Wang, Luyu and Zhou, Guangyao and Viola, Fabio},
	year = {2020},
}

@inproceedings{xu_eclipse_2024,
	title = {{ECLipsE}: {Efficient} {Compositional} {Lipschitz} {Constant} {Estimation} for {Deep} {Neural} {Networks}},
	volume = {37},
	url = {https://proceedings.neurips.cc/paper_files/paper/2024/file/1419d8554191a65ea4f2d8e1057973e4-Paper-Conference.pdf},
	booktitle = {Advances in {Neural} {Information} {Processing} {Systems}},
	publisher = {Curran Associates, Inc.},
	author = {Xu, Yuezhu and Sivaranjani, S.},
	editor = {Globerson, A. and Mackey, L. and Belgrave, D. and Fan, A. and Paquet, U. and Tomczak, J. and Zhang, C.},
	year = {2024},
	pages = {10414--10441},
}

@article{zhang_recurjac_2019,
	title = {{RecurJac}: {An} {Efficient} {Recursive} {Algorithm} for {Bounding} {Jacobian} {Matrix} of {Neural} {Networks} and {Its} {Applications}},
	volume = {33},
	copyright = {https://www.aaai.org},
	issn = {2374-3468, 2159-5399},
	shorttitle = {{RecurJac}},
	url = {https://ojs.aaai.org/index.php/AAAI/article/view/4522},
	doi = {10.1609/aaai.v33i01.33015757},
	number = {01},
	urldate = {2025-10-24},
	journal = {Proceedings of the AAAI Conference on Artificial Intelligence},
	author = {Zhang, Huan and Zhang, Pengchuan and Hsieh, Cho-Jui},
	month = jul,
	year = {2019},
	pages = {5757--5764},
}

@inproceedings{fazlyab_efficient_2019,
	title = {Efficient and {Accurate} {Estimation} of {Lipschitz} {Constants} for {Deep} {Neural} {Networks}},
	volume = {32},
	url = {https://proceedings.neurips.cc/paper_files/paper/2019/hash/95e1533eb1b20a97777749fb94fdb944-Abstract.html},
	urldate = {2025-10-24},
	booktitle = {Advances in {Neural} {Information} {Processing} {Systems}},
	publisher = {Curran Associates, Inc.},
	author = {Fazlyab, Mahyar and Robey, Alexander and Hassani, Hamed and Morari, Manfred and Pappas, George},
	year = {2019},
}

@inproceedings{leino_globally-robust_2021,
	title = {Globally-robust neural networks},
	booktitle = {International {Conference} on {Machine} {Learning}},
	publisher = {PMLR},
	author = {Leino, Klas and Wang, Zifan and Fredrikson, Matt},
	year = {2021},
	pages = {6212--6222},
}

@misc{szegedy_intriguing_2014,
	title = {Intriguing properties of neural networks},
	url = {http://arxiv.org/abs/1312.6199},
	doi = {10.48550/arXiv.1312.6199},
	urldate = {2025-12-11},
	publisher = {arXiv},
	author = {Szegedy, Christian and Zaremba, Wojciech and Sutskever, Ilya and Bruna, Joan and Erhan, Dumitru and Goodfellow, Ian and Fergus, Rob},
	month = feb,
	year = {2014},
	note = {arXiv:1312.6199 [cs]},
}

@article{gouk_regularisation_2021,
	title = {Regularisation of neural networks by enforcing {Lipschitz} continuity},
	volume = {110},
	issn = {0885-6125, 1573-0565},
	url = {http://link.springer.com/10.1007/s10994-020-05929-w},
	doi = {10.1007/s10994-020-05929-w},
	language = {en},
	number = {2},
	urldate = {2025-12-11},
	journal = {Machine Learning},
	author = {Gouk, Henry and Frank, Eibe and Pfahringer, Bernhard and Cree, Michael J.},
	month = feb,
	year = {2021},
	pages = {393--416},
}

@inproceedings{anil_sorting_2019,
	series = {Proceedings of {Machine} {Learning} {Research}},
	title = {Sorting {Out} {Lipschitz} {Function} {Approximation}},
	volume = {97},
	url = {https://proceedings.mlr.press/v97/anil19a.html},
	booktitle = {Proceedings of the 36th {International} {Conference} on {Machine} {Learning}},
	publisher = {PMLR},
	author = {Anil, Cem and Lucas, James and Grosse, Roger},
	editor = {Chaudhuri, Kamalika and Salakhutdinov, Ruslan},
	month = jun,
	year = {2019},
	pages = {291--301},
}

@inproceedings{bartlett_spectrally-normalized_2017,
	address = {Red Hook, NY, USA},
	series = {{NIPS}'17},
	title = {Spectrally-normalized margin bounds for neural networks},
	isbn = {978-1-5108-6096-4},
	booktitle = {Proceedings of the 31st {International} {Conference} on {Neural} {Information} {Processing} {Systems}},
	publisher = {Curran Associates Inc.},
	author = {Bartlett, Peter L. and Foster, Dylan J. and Telgarsky, Matus},
	year = {2017},
	pages = {6241--6250},
}

@inproceedings{virmaux_lipschitz_2018,
	title = {Lipschitz {Regularity} of {Deep} {Neural} {Networks}: {Analysis} and {Efficient} {Estimation}},
	volume = {31},
	url = {https://proceedings.neurips.cc/paper_files/paper/2018/file/d54e99a6c03704e95e6965532dec148b-Paper.pdf},
	booktitle = {Advances in {Neural} {Information} {Processing} {Systems}},
	publisher = {Curran Associates, Inc.},
	author = {Virmaux, Aladin and Scaman, Kevin},
	editor = {Bengio, S. and Wallach, H. and Larochelle, H. and Grauman, K. and Cesa-Bianchi, N. and Garnett, R.},
	year = {2018},
}

@inproceedings{jordan_exactly_2020,
	title = {Exactly {Computing} the {Local} {Lipschitz} {Constant} of {ReLU} {Networks}},
	volume = {33},
	url = {https://proceedings.neurips.cc/paper_files/paper/2020/file/5227fa9a19dce7ba113f50a405dcaf09-Paper.pdf},
	booktitle = {Advances in {Neural} {Information} {Processing} {Systems}},
	publisher = {Curran Associates, Inc.},
	author = {Jordan, Matt and Dimakis, Alexandros G.},
	editor = {Larochelle, H. and Ranzato, M. and Hadsell, R. and Balcan, M. F. and Lin, H.},
	year = {2020},
	pages = {7344--7353},
}

@inproceedings{chen_semialgebraic_2020,
	title = {Semialgebraic {Optimization} for {Lipschitz} {Constants} of {ReLU} {Networks}},
	volume = {33},
	url = {https://proceedings.neurips.cc/paper_files/paper/2020/file/dea9ddb25cbf2352cf4dec30222a02a5-Paper.pdf},
	booktitle = {Advances in {Neural} {Information} {Processing} {Systems}},
	publisher = {Curran Associates, Inc.},
	author = {Chen, Tong and Lasserre, Jean B. and Magron, Victor and Pauwels, Edouard},
	editor = {Larochelle, H. and Ranzato, M. and Hadsell, R. and Balcan, M. F. and Lin, H.},
	year = {2020},
	pages = {19189--19200},
}

@inproceedings{golowich_size-independent_2018,
	series = {Proceedings of {Machine} {Learning} {Research}},
	title = {Size-{Independent} {Sample} {Complexity} of {Neural} {Networks}},
	volume = {75},
	url = {https://proceedings.mlr.press/v75/golowich18a.html},
	booktitle = {Proceedings of the 31st {Conference} {On} {Learning} {Theory}},
	publisher = {PMLR},
	author = {Golowich, Noah and Rakhlin, Alexander and Shamir, Ohad},
	editor = {Bubeck, Sébastien and Perchet, Vianney and Rigollet, Philippe},
	year = {2018},
	pages = {297--299},
}

@misc{neyshabur_data-dependent_2016,
	title = {Data-{Dependent} {Path} {Normalization} in {Neural} {Networks}},
	url = {https://arxiv.org/abs/1511.06747},
	author = {Neyshabur, Behnam and Tomioka, Ryota and Salakhutdinov, Ruslan and Srebro, Nathan},
	year = {2016},
	note = {\_eprint: 1511.06747},
}

@inproceedings{neyshabur_norm-based_2015,
	title = {Norm-based capacity control in neural networks},
	booktitle = {Conference on learning theory},
	publisher = {PMLR},
	author = {Neyshabur, Behnam and Tomioka, Ryota and Srebro, Nathan},
	year = {2015},
	pages = {1376--1401},
}

@inproceedings{neyshabur_path-sgd_2015,
	title = {Path-{SGD}: {Path}-{Normalized} {Optimization} in {Deep} {Neural} {Networks}},
	volume = {28},
	shorttitle = {Path-{SGD}},
	url = {https://proceedings.neurips.cc/paper_files/paper/2015/hash/eaa32c96f620053cf442ad32258076b9-Abstract.html},
	urldate = {2025-12-17},
	booktitle = {Advances in {Neural} {Information} {Processing} {Systems}},
	publisher = {Curran Associates, Inc.},
	author = {Neyshabur, Behnam and Salakhutdinov, Russ R and Srebro, Nati},
	year = {2015},
}

@inproceedings{neyshabur_pac-bayesian_2018,
	title = {A {PAC}-{Bayesian} {Approach} to {Spectrally}-{Normalized} {Margin} {Bounds} for {Neural} {Networks}},
	url = {https://openreview.net/forum?id=Skz_WfbCZ},
	booktitle = {International {Conference} on {Learning} {Representations}},
	author = {Neyshabur, Behnam and Bhojanapalli, Srinadh and Srebro, Nathan},
	year = {2018},
}

@article{shi_efficiently_2022,
	title = {Efficiently computing local lipschitz constants of neural networks via bound propagation},
	volume = {35},
	journal = {Advances in Neural Information Processing Systems},
	author = {Shi, Zhouxing and Wang, Yihan and Zhang, Huan and Kolter, J Zico and Hsieh, Cho-Jui},
	year = {2022},
	pages = {2350--2364},
}

@phdthesis{bethune_deep_2024,
	type = {Theses},
	title = {Deep learning with {Lipschitz} constraints},
	url = {https://theses.hal.science/tel-04674274},
	school = {Université de Toulouse},
	author = {Béthune, Louis},
	month = feb,
	year = {2024},
	note = {Issue: 2024TLSES014},
}

@inproceedings{wang_direct_2023,
	title = {Direct parameterization of lipschitz-bounded deep networks},
	booktitle = {International {Conference} on {Machine} {Learning}},
	publisher = {PMLR},
	author = {Wang, Ruigang and Manchester, Ian},
	year = {2023},
	pages = {36093--36110},
}

@article{xu_learning_2023,
	title = {Learning {Dissipative} {Neural} {Dynamical} {Systems}},
	volume = {7},
	issn = {2475-1456},
	url = {https://ieeexplore.ieee.org/abstract/document/10335923},
	doi = {10.1109/LCSYS.2023.3337851},
	urldate = {2025-12-18},
	journal = {IEEE Control Systems Letters},
	author = {Xu, Yuezhu and Sivaranjani, S.},
	year = {2023},
	pages = {3531--3536},
}

@article{zhang_efficient_2018,
	title = {Efficient neural network robustness certification with general activation functions},
	volume = {31},
	journal = {Advances in neural information processing systems},
	author = {Zhang, Huan and Weng, Tsui-Wei and Chen, Pin-Yu and Hsieh, Cho-Jui and Daniel, Luca},
	year = {2018},
}

@inproceedings{nam_position_2025,
	title = {Position: {Solve} {Layerwise} {Linear} {Models} {First} to {Understand} {Neural} {Dynamical} {Phenomena} ({Neural} {Collapse}, {Emergence}, {Lazy}/{Rich} {Regime}, and {Grokking})},
	shorttitle = {Position},
	url = {https://openreview.net/forum?id=nrlGUdlo16},
	language = {en},
	urldate = {2026-01-13},
	author = {Nam, Yoonsoo and Lee, Seok Hyeong and Dominé, Clémentine Carla Juliette and Park, Yeachan and London, Charles and Choi, Wonyl and Göring, Niclas Alexander and Lee, Seungjai},
	month = jun,
	year = {2025},
}

@misc{shi_neural_2025,
	title = {Neural {Network} {Verification} with {Branch}-and-{Bound} for {General} {Nonlinearities}},
	url = {http://arxiv.org/abs/2405.21063},
	doi = {10.48550/arXiv.2405.21063},
	urldate = {2026-01-21},
	publisher = {arXiv},
	author = {Shi, Zhouxing and Jin, Qirui and Kolter, Zico and Jana, Suman and Hsieh, Cho-Jui and Zhang, Huan},
	month = feb,
	year = {2025},
	note = {arXiv:2405.21063 [cs]},
}

@inproceedings{xue_chordal_2022,
	title = {Chordal sparsity for lipschitz constant estimation of deep neural networks},
	booktitle = {2022 {IEEE} 61st {Conference} on {Decision} and {Control} ({CDC})},
	publisher = {IEEE},
	author = {Xue, Anton and Lindemann, Lars and Robey, Alexander and Hassani, Hamed and Pappas, George J and Alur, Rajeev},
	year = {2022},
	pages = {3389--3396},
}

@inproceedings{bhowmick_lipbab_2021,
	address = {Cham},
	title = {{LipBaB}: {Computing} {Exact} {Lipschitz} {Constant} of {ReLU} {Networks}},
	isbn = {978-3-030-86380-7},
	shorttitle = {{LipBaB}},
	doi = {10.1007/978-3-030-86380-7_13},
	language = {en},
	booktitle = {Artificial {Neural} {Networks} and {Machine} {Learning} – {ICANN} 2021},
	publisher = {Springer International Publishing},
	author = {Bhowmick, Aritra and D’Souza, Meenakshi and Raghavan, G. Srinivasa},
	editor = {Farkaš, Igor and Masulli, Paolo and Otte, Sebastian and Wermter, Stefan},
	year = {2021},
	pages = {151--162},
}

@article{pauli_training_2021,
	title = {Training robust neural networks using {Lipschitz} bounds},
	volume = {6},
	journal = {IEEE Control Systems Letters},
	publisher = {IEEE},
	author = {Pauli, Patricia and Koch, Anne and Berberich, Julian and Kohler, Paul and Allgöwer, Frank},
	year = {2021},
	pages = {121--126},
}

@misc{sbihi_miqcqp_2024,
	title = {{MIQCQP} reformulation of the {ReLU} neural networks {Lipschitz} constant estimation problem},
	url = {http://arxiv.org/abs/2402.01199},
	doi = {10.48550/arXiv.2402.01199},
	urldate = {2026-01-28},
	publisher = {arXiv},
	author = {Sbihi, Mohammed and Jan, Sophie and Couellan, Nicolas},
	month = feb,
	year = {2024},
	note = {arXiv:2402.01199 [math]},
}

@misc{araujo_unified_2023,
	title = {A {Unified} {Algebraic} {Perspective} on {Lipschitz} {Neural} {Networks}},
	url = {http://arxiv.org/abs/2303.03169},
	doi = {10.48550/arXiv.2303.03169},
	urldate = {2026-04-01},
	publisher = {arXiv},
	author = {Araujo, Alexandre and Havens, Aaron and Delattre, Blaise and Allauzen, Alexandre and Hu, Bin},
	month = oct,
	year = {2023},
	note = {arXiv:2303.03169 [cs]},
}

@misc{latorre_lipschitz_2020,
	title = {Lipschitz constant estimation of {Neural} {Networks} via sparse polynomial optimization},
	url = {http://arxiv.org/abs/2004.08688},
	doi = {10.48550/arXiv.2004.08688},
	urldate = {2026-05-04},
	publisher = {arXiv},
	author = {Latorre, Fabian and Rolland, Paul and Cevher, Volkan},
	month = apr,
	year = {2020},
	note = {arXiv:2004.08688 [cs]},
}

@article{lecun_mnist_2010,
	title = {{MNIST} handwritten digit database},
	volume = {2},
	journal = {ATT Labs [Online]. Available: http://yann.lecun.com/exdb/mnist},
	author = {LeCun, Yann and Cortes, Corinna and Burges, CJ},
	year = {2010},
}

@misc{shewchuk_geometry_2024,
	title = {The {Geometry} of the {Set} of {Equivalent} {Linear} {Neural} {Networks}},
	url = {http://arxiv.org/abs/2404.14855},
	doi = {10.48550/arXiv.2404.14855},
	urldate = {2026-07-02},
	publisher = {arXiv},
	author = {Shewchuk, Jonathan Richard and Bhattacharya, Sagnik},
	month = apr,
	year = {2024},
	note = {arXiv:2404.14855 [cs.LG]},
}

@article{manchester_neural_2026,
	title = {Neural {Networks} in the {Loop}: {Learning} with {Stability} and {Robustness} {Guarantees}},
	volume = {9},
	copyright = {https://creativecommons.org/licenses/by-nc-nd/4.0},
	issn = {2573-5144},
	shorttitle = {Neural {Networks} in the {Loop}},
	url = {https://www.annualreviews.org/content/journals/10.1146/annurev-control-022624-032459},
	doi = {10.1146/annurev-control-022624-032459},
	language = {en},
	number = {1},
	urldate = {2026-08-27},
	journal = {Annual Review of Control, Robotics, and Autonomous Systems},
	author = {Manchester, Ian R. and Wang, Ruigang and Barbara, Nicholas H.},
	month = may,
	year = {2026},
	pages = {181--214},
}

@article{bubeck_universal_2023,
	title = {A {Universal} {Law} of {Robustness} via {Isoperimetry}},
	volume = {70},
	issn = {0004-5411},
	url = {https://dl.acm.org/doi/10.1145/3578580},
	doi = {10.1145/3578580},
	number = {2},
	urldate = {2026-08-27},
	journal = {Journal of the ACM},
	author = {Bubeck, Sébastien and Sellke, Mark},
	month = mar,
	year = {2023},
	pages = {10:1--10:18},
}

@inproceedings{prach_almost-orthogonal_2022,
	address = {Cham},
	title = {Almost-{Orthogonal} {Layers} for {Efficient} {General}-{Purpose} {Lipschitz} {Networks}},
	isbn = {978-3-031-19803-8},
	doi = {10.1007/978-3-031-19803-8_21},
	language = {en},
	booktitle = {Computer {Vision} – {ECCV} 2022},
	publisher = {Springer Nature Switzerland},
	author = {Prach, Bernd and Lampert, Christoph H.},
	editor = {Avidan, Shai and Brostow, Gabriel and Cissé, Moustapha and Farinella, Giovanni Maria and Hassner, Tal},
	year = {2022},
	pages = {350--365},
}

@inproceedings{meunier_dynamical_2022,
	title = {A {Dynamical} {System} {Perspective} for {Lipschitz} {Neural} {Networks}},
	issn = {2640-3498},
	url = {https://proceedings.mlr.press/v162/meunier22a.html},
	language = {en},
	urldate = {2026-08-27},
	booktitle = {Proceedings of the 39th {International} {Conference} on {Machine} {Learning}},
	publisher = {PMLR},
	author = {Meunier, Laurent and Delattre, Blaise J. and Araujo, Alexandre and Allauzen, Alexandre},
	month = jun,
	year = {2022},
	pages = {15484--15500},
}

@article{tseng_convergence_1993,
	title = {On the convergence of the exponential multiplier method for convex programming},
	volume = {60},
	issn = {1436-4646},
	url = {https://doi.org/10.1007/BF01580598},
	doi = {10.1007/BF01580598},
	language = {en},
	number = {1},
	urldate = {2026-09-04},
	journal = {Mathematical Programming},
	author = {Tseng, Paul and Bertsekas, Dimitri P.},
	month = jun,
	year = {1993},
	pages = {1--19},
}

@inproceedings{kort_new_1972,
	title = {A new penalty function method for constrained minimization},
	url = {https://ieeexplore.ieee.org/abstract/document/4044894},
	doi = {10.1109/CDC.1972.268971},
	urldate = {2026-09-04},
	booktitle = {Proceedings of the 1972 {IEEE} {Conference} on {Decision} and {Control} and 11th {Symposium} on {Adaptive} {Processes}},
	author = {Kort, Barry W. and Bertsekas, Dimitri P.},
	month = dec,
	year = {1972},
	pages = {162--166},
}

@inproceedings{oliveira_convergence_2026,
	title = {On the {Convergence} of {Overparameterized} {Problems}: {Inherent} {Properties} of the {Compositional} {Structure} of {Neural} {Networks}},
	issn = {2640-3498},
	shorttitle = {On the {Convergence} of {Overparameterized} {Problems}},
	url = {https://proceedings.mlr.press/v331/oliveira26a.html},
	language = {en},
	urldate = {2026-10-06},
	booktitle = {Proceedings of {The} 8th {Annual} {Learning} for {Dynamics} and {Control} {Conference}},
	publisher = {PMLR},
	author = {Oliveira, Arthur Castello Branco de and Jatkar, Dhruv D. and Sontag, Eduardo},
	month = jun,
	year = {2026},
	pages = {1088--1107},
}

@inproceedings{shewchuk_hierarchy_2026,
	address = {Dagstuhl, Germany},
	series = {Leibniz {International} {Proceedings} in {Informatics} ({LIPIcs})},
	title = {The {Hierarchy} of {Manifolds} in a {Stratification} of the {Set} of {Equivalent} {Linear} {Neural} {Networks}},
	volume = {367},
	isbn = {978-3-95977-418-5},
	issn = {1868-8969},
	url = {https://drops.dagstuhl.de/entities/document/10.4230/LIPIcs.SoCG.2026.91},
	doi = {10.4230/LIPIcs.SoCG.2026.91},
	urldate = {2026-10-06},
	booktitle = {42nd {International} {Symposium} on {Computational} {Geometry} ({SoCG} 2026)},
	publisher = {Schloss Dagstuhl – Leibniz-Zentrum für Informatik},
	author = {Shewchuk, Jonathan Richard and Bhattacharya, Sagnik},
	editor = {Ahn, Hee-Kap and Hoffmann, Michael and Nayyeri, Amir},
	year = {2026},
	pages = {91:1--91:19},
}

@misc{zhang_depth_2020,
	title = {Depth creates no more spurious local minima},
	url = {http://arxiv.org/abs/1901.09827},
	doi = {10.48550/arXiv.1901.09827},
	urldate = {2026-10-06},
	publisher = {arXiv},
	author = {Zhang, Li},
	month = jan,
	year = {2020},
	note = {arXiv:1901.09827 [cs.LG]},
}

@inproceedings{arjovsky_wasserstein_2017,
	title = {Wasserstein {Generative} {Adversarial} {Networks}},
	issn = {2640-3498},
	url = {https://proceedings.mlr.press/v70/arjovsky17a.html},
	language = {en},
	urldate = {2026-11-09},
	booktitle = {Proceedings of the 34th {International} {Conference} on {Machine} {Learning}},
	publisher = {PMLR},
	author = {Arjovsky, Martin and Chintala, Soumith and Bottou, Léon},
	month = jul,
	year = {2017},
	pages = {214--223},
}

@misc{biertumpfel_lipschitz_2026,
	title = {Lipschitz {Bounds} for {Neural} {Networks} {Using} a {Two}-{Copy} {Linear} {Fractional} {Transformation}},
	url = {http://arxiv.org/abs/2610.22692},
	doi = {10.48550/arXiv.2610.22692},
	urldate = {2026-11-09},
	publisher = {arXiv},
	author = {Biertümpfel, Felix and Hu, Bin and Seiler, Peter},
	month = oct,
	year = {2026},
	note = {arXiv:2610.22692 [eess.SY]},
}

@article{cybenko_approximation_1989,
	title = {Approximation by superpositions of a sigmoidal function},
	volume = {2},
	issn = {1435-568X},
	url = {https://doi.org/10.1007/BF02551274},
	doi = {10.1007/BF02551274},
	language = {en},
	number = {4},
	urldate = {2026-11-10},
	journal = {Mathematics of Control, Signals and Systems},
	author = {Cybenko, G.},
	month = dec,
	year = {1989},
	pages = {303--314},
}

@article{funahashi_approximate_1989,
	title = {On the approximate realization of continuous mappings by neural networks},
	volume = {2},
	issn = {0893-6080},
	url = {https://www.sciencedirect.com/science/article/pii/0893608089900038},
	doi = {10.1016/0893-6080(89)90003-8},
	number = {3},
	urldate = {2026-11-10},
	journal = {Neural Networks},
	author = {Funahashi, Ken-Ichi},
	month = jan,
	year = {1989},
	pages = {183--192},
}

@article{hornik_multilayer_1989,
	title = {Multilayer feedforward networks are universal approximators},
	volume = {2},
	issn = {0893-6080},
	url = {https://www.sciencedirect.com/science/article/pii/0893608089900208},
	doi = {10.1016/0893-6080(89)90020-8},
	number = {5},
	urldate = {2026-11-10},
	journal = {Neural Networks},
	author = {Hornik, Kurt and Stinchcombe, Maxwell and White, Halbert},
	month = jan,
	year = {1989},
	pages = {359--366},
}

@article{leshno_multilayer_1993,
	title = {Multilayer feedforward networks with a nonpolynomial activation function can approximate any function},
	volume = {6},
	issn = {0893-6080},
	url = {https://www.sciencedirect.com/science/article/pii/S0893608005801315},
	doi = {10.1016/S0893-6080(05)80131-5},
	number = {6},
	urldate = {2026-11-10},
	journal = {Neural Networks},
	author = {Leshno, Moshe and Lin, Vladimir Ya. and Pinkus, Allan and Schocken, Shimon},
	month = jan,
	year = {1993},
	pages = {861--867},
}

@incollection{berner_modern_2022,
	title = {The {Modern} {Mathematics} of {Deep} {Learning}},
	url = {http://arxiv.org/abs/2105.04026},
	doi = {10.1017/9781009025096.002},
	urldate = {2026-11-10},
	author = {Berner, Julius and Grohs, Philipp and Kutyniok, Gitta and Petersen, Philipp},
	month = dec,
	year = {2022},
	note = {arXiv:2105.04026 [cs.LG]},
	pages = {1--111},
}

@article{xu_eclipse-gen-local_2026,
	title = {{ECLipsE}-{Gen}-{Local}: {Efficient} {Compositional} {Local} {Lipschitz} {Estimates} for {Deep} {Neural} {Networks}},
	issn = {2835-8856},
	url = {https://openreview.net/forum?id=CuqnFjeu5a},
	journal = {Transactions on Machine Learning Research},
	author = {Xu, Yuezhu and Sivaranjani, S.},
	year = {2026},
}

\newpage
\appendix
\tableofcontents
\listoffigures
\listoftables

\section{ECLipsE as a Schur-complement decomposition of LipSDP}
\label{sec:eclipse-lipsdp-derivation}
ECLipsE \citep{xu_eclipse_2024} begins from the same monotonicity certificate as LipSDP \citep{fazlyab_efficient_2019}, but uses the block structure of that certificate to construct a sequence of quadratic metrics.
Let \(x^\ell\) and \(\tilde x^\ell\) be the activations generated by two inputs, and set \(\Delta x^\ell=x^\ell-\tilde x^\ell\).
For a positive semidefinite matrix \(M\), write \(\|u\|_M^2=u^\top M u\).
Starting from \(M^0=I\), ECLipsE constructs quadratic metrics \(M^\ell\succ0\) satisfying
\begin{align}
  \|\Delta x^\ell\|_{M^\ell}^2
  \leq
  \|\Delta x^{\ell-1}\|_{M^{\ell-1}}^2.
  \label{eq:eclipse-layerwise-iqc}
\end{align}
These inequalities telescope:
\begin{align}
  \|\Delta x^{L-1}\|_{M^{L-1}}^2
  \leq
  \|\Delta x^{L-2}\|_{M^{L-2}}^2
  \leq
  \cdots
  \leq
  \|\Delta x^0\|_{M^0}^2
  =
  \|\Delta x^0\|_2^2.
  \label{eq:eclipse-telescope-hidden}
\end{align}
If the final readout is linear, \(\Delta x^L=W^L\Delta x^{L-1}\), then we find a scalar \(F>0\) satisfying
\begin{align}
  M^{L-1}-F(W^L)^\top W^L \succ 0.
  \label{eq:eclipse-final-schur}
\end{align}
Thus \(F\|\Delta x^L\|_2^2\leq\|\Delta x^0\|_2^2\), so
\(\|f(x)-f(\tilde x)\|_2\leq F^{-1/2}\|x-\tilde x\|_2\).

Suppose the elementwise nonlinearity is slope-restricted in \([\alpha,\beta]\), and define
\begin{align}
  p=\alpha\beta,
  \qquad
  m=\frac{\alpha+\beta}{2}.
\end{align}
Then for each hidden layer
\begin{align}
  \Delta v^\ell = W^\ell \Delta x^{\ell-1},
  \qquad
  \Delta x^\ell
  =
  \sigma^\ell(v^\ell)-\sigma^\ell(\tilde v^\ell)
\end{align}
satisfies the incremental sector inequality
\begin{align}
  \begin{bmatrix}
    \Delta v^\ell \\
    \Delta x^\ell
  \end{bmatrix}^{\!\top}
  \begin{bmatrix}
    p\Lambda_\ell & -m\Lambda_\ell \\
    -m\Lambda_\ell & \Lambda_\ell
  \end{bmatrix}
  \begin{bmatrix}
    \Delta v^\ell \\
    \Delta x^\ell
  \end{bmatrix} 
  \leq 0,
  \qquad
  \Lambda_\ell \in \mathbb D_+,
  \label{eq:sector-iqc}
\end{align}
where \(\mathbb D_+\) denotes the cone of nonnegative diagonal matrices.
Equivalently,
\begin{align}
  \|\Delta x^\ell\|_{\Lambda_\ell}^2
  +
  p\|W^\ell\Delta x^{\ell-1}\|_{\Lambda_\ell}^2
  \leq
  2m(\Delta x^\ell)^\top \Lambda_\ell W^\ell \Delta x^{\ell-1}.
  \label{eq:sector-expanded}
\end{align}
The right-hand side can be bounded by completing the square.
Given \(M^{\ell-1}\succ0\), choose \(\Lambda_\ell\) so that
\begin{align}
  A^\ell
  =
  M^{\ell-1}
  +
  p(W^\ell)^\top\Lambda_\ell W^\ell
  \succ0.
  \label{eq:eclipse-a-ell}
\end{align}
Then
\begin{align}
  2m(\Delta x^\ell)^\top \Lambda_\ell W^\ell \Delta x^{\ell-1}
  \leq
  \|\Delta x^{\ell-1}\|_{A^\ell}^2
  +
  m^2
  \left\|\Delta x^\ell\right\|_{\Lambda_\ell W^\ell (A^\ell)^{-1}(W^\ell)^\top \Lambda_\ell}^2.
\end{align}
Combining this with \eqref{eq:sector-expanded}, and using
\(\|\Delta x^{\ell-1}\|_{A^\ell}^2-p\|W^\ell\Delta x^{\ell-1}\|_{\Lambda_\ell}^2
=\|\Delta x^{\ell-1}\|_{M^{\ell-1}}^2\), gives
\begin{align}
  \|\Delta x^\ell\|_{M^\ell}^2
  \leq
  \|\Delta x^{\ell-1}\|_{M^{\ell-1}}^2,
  \qquad
  M^\ell
  =
  \Lambda_\ell
  -
  m^2\Lambda_\ell W^\ell (A^\ell)^{-1}(W^\ell)^\top\Lambda_\ell.
\end{align}
This is the layerwise meaning of the ECLipsE decomposition.
LipSDP asks for all multipliers \(\Lambda_\ell\) at once through one large block-tridiagonal semidefinite constraint.
ECLipsE applies the corresponding block Cholesky or Schur-complement recursion: starting with \(M^0=I\), choose \(\Lambda_\ell\) so that \(A^\ell\succ0\) and \(M^\ell\succ0\), and pass \(M^\ell\) forward as the metric for the next layer.
The matrices \(A^\ell\) are the successive Schur-complement pivots of the LipSDP matrix, while \(M^\ell\) is the part of the next pivot that remains after removing the next sector term.
When \((\alpha,\beta)=(0,1)\), we have \(p=0\) and \(m=1/2\), so \(A^\ell=M^{\ell-1}\) and the recursion reduces to the familiar ECLipsE formula
\begin{align}
  M^\ell
  =
  \Lambda_\ell
  -
  \frac{1}{4}\Lambda_\ell W^\ell (M^{\ell-1})^{-1}(W^\ell)^\top\Lambda_\ell.
\end{align}

The point of constructing the \(M^\ell\)'s is that \eqref{eq:eclipse-layerwise-iqc} telescopes.
For the hidden layers, this gives \eqref{eq:eclipse-telescope-hidden}.
If the final readout is linear, then \eqref{eq:eclipse-final-schur} implies
\begin{align}
  F\|\Delta x^L\|_2^2
  =
  F\|W^L\Delta x^{L-1}\|_2^2
  \leq
  \|\Delta x^{L-1}\|_{M^{L-1}}^2
  \leq
  \|\Delta x^0\|_2^2.
\end{align}
Consequently \(\|f(x)-f(\tilde x)\|_2\leq F^{-1/2}\|x-\tilde x\|_2\).
In the notation of \eqref{eq:eclipse-layerwise-iqc}, this is the same as taking the output metric to be \(M^L=FI\): the learned intermediate metrics certify a nonincreasing chain of quadratic distances from input to output, and the hidden-layer metrics cancel rather than being multiplied as scalar layer norms.

\section{Proofs}
\label{sec:proofs}
\begin{proof}[Proof of \Cref{thm:trivial-bound-tight-asl-sla-las} (ASL)]
The Jacobian of the ASL unit \(x \mapsto \sigma(Wx +b)\) is
\begin{gather*}
  x \mapsto \operatorname{diag}\left(\sigma'(Wx +b)\right) W.
\end{gather*}
Suppose that \(\sigma\) achieves its steepest slope at \(z_0\).
For generic \(W\) and \(b\),
there exists a preimage \(x_0\) such that \(W x_0 + b = z_0 \bm{1}\).
At \(x = x_0\), \(\operatorname{diag}\left(\sigma'(Wx +b)\right)\) becomes the scalar matrix \(\sigma'(z_0) I\), and the Jacobian equals \(\sigma'(z_0) W\).
The trivial Lipschitz bound is tight by the homogeneity of induced matrix norm.
\end{proof}
\begin{proof}[Proof of \Cref{thm:trivial-bound-tight-asl-sla-las} (SLA)]
The Jacobian of the SLA unit \(x \mapsto W\sigma(x) + b\)
is
\begin{gather*}
  x \mapsto W \operatorname{diag}(\sigma'(x)),
\end{gather*}
which achieves the trivial Lipschitz bound in the \(p\)-norm at \(x = z_0 \bm{1}\).
\end{proof}
\begin{proof}[Proof of \Cref{thm:trivial-bound-tight-asl-sla-las} (LAS)]
  The Jacobian of the LAS unit \(x \mapsto W \sigma(x + b)\) is
  \begin{align*}
    x \mapsto W \operatorname{diag}\left(\sigma'(x + b)\right),
  \end{align*}
  which achieves the trivial Lipschitz bound in the \(p\)-norm at \(x = z_0 \bm{1} - b\).
\end{proof}

\begin{proof}[Proof of \Cref{thm:lasl-biases-relu-tanh} (ReLU)]
  For even \(d\), consider the neural network \(\mathbb R \to \mathbb R\) defined by
  \begin{align*}
    x \mapsto \sum_{i = 1}^d (-1)^i \operatorname{ReLU}(x + b_i) = W^1 \operatorname{ReLU}(W^0 x + b)
  \end{align*}
  where 
  \begin{align*}
    W^1 &= (-1, 1, -1, \ldots) \\
    W^0 &= (1, 1, 1, \ldots)^\intercal \\
    b &= (1, 1, \ldots)
  \end{align*}
  The true Lipschitz constant is 0, but the trivial 2-Lipschitz bound is
  \begin{align*}
    \underbrace{\left\|W^1\right\|_2}_{= \sqrt{d}} \underbrace{\left\|W^0 \right\|_2}_{= \sqrt d} = d.
  \end{align*}
  On the other hand, if the bias is changed to 
  \begin{align*}
    b' = (1, -1, 1, -1, \ldots),
  \end{align*}
  then the true Lipschitz constant is \(\Omega(d)\).
\end{proof}
\begin{proof}[Proof of \Cref{thm:lasl-biases-relu-tanh} (tanh)]
  Set \(m=\lfloor d/2\rfloor\) and take \(m\) pairs of neurons with input weights \(1,1\),
  output weights \(1,-1\), and biases \(0,t\), respectively. Give any unused neuron zero weights.
  The resulting map is \(f_t(x)=m(\tanh x-\tanh(x+t))\).
  Both weight matrices have spectral norm \(\sqrt{2m}\), so the trivial bound is \(2m\) for every \(t\).

  At \(t=0\) the pairs cancel identically. At \(t=1\), the true Lipschitz constant is at least
  \(f_1'(0)=m(1-\tanh'(1))=m\tanh^2(1)=\Omega(d)\).
\end{proof}
\begin{proof}[Proof of Theorem~\ref{thm:lasl-biases-sin}]
  Consider the LASL unit \(f_b : \mathbb R \to \mathbb R\) defined by
  \begin{align*}
    f_b(x) &= \bm{1}^\intercal \sin(\bm{1} x + b) = \sum_{i = 1}^{2d} \sin(x + b_i),
  \end{align*}
  where \(\bm{1} \in \mathbb R^{2d}\) is the vector of ones and \(b \in \mathbb R^{2d}\) is a bias vector.
  The trivial 2-Lipschitz bound equals
  \begin{align*}
    \left\|\bm{1}^\intercal\right\|_2\, \left\|\bm{1}\right\|_2 = \left\|\bm{1}\right\|_2^2 = 2d = \Omega(d),
  \end{align*}
  since \(|\sin'(\cdot)| \le 1\).
  
  On one hand, if we take \(b = 0\), the true Lipschitz constant is \(2d\), and the trivial bound is tight.
  
  On the other hand, if we take \(b_i = i\pi\) for all \(i\), then each sine term cancels out with its neighbor, resulting in a constant function with Lipschitz constant 0.
\end{proof}

\begin{proof}[Proof of \cref{thm:condition-number-tightness-product-identity}]
  \ref{item:K-equals-1} \(\iff\) \ref{item:minimize-K} follows from the fact that \(K=1\) is achievable, and \(K\) cannot be less than 1:
  \begin{align*}
    1 &= \left\|I\right\| = \left\|W_L W_{L-1} \cdots W_1\right\| \\
    &\le \left\|W_L\right\| \left\|W_{L-1}\right\| \cdots \left\|W_1\right\| = K.
  \end{align*}

  \ref{item:minimize-K} \(\implies\) \ref{item:condition-number-1} by the calculation
  \begin{align*}
    1 = K &\geq \left\|W_L\right\| \left\|W_{L-1} \cdots W_1\right\| 
    \\
    &= \left\|W_L\right\| \left\|W_L^{-1}\right\| 
    = \kappa(W_L)
  \end{align*}
  and induction on \(L\).
  
  \ref{item:condition-number-1} \(\implies\) \ref{item:K-equals-1} by the fact that each \(W_i\), \(i \in \{1, \ldots, L\}\), is a scalar multiple of an orthogonal matrix, and the product of orthogonal matrices is orthogonal.
\end{proof}

\section{MNIST diagonal reparameterization experiment details}
\label{subsec:mnist-gauge-methodology}
The diagonal reparameterization experiment is implemented in the notebook \texttt{MNIST-ReLU-gauge-dead.ipynb}.
Inputs are normalized MNIST \citep[MIT license]{lecun_mnist_2010} images flattened into \(\mathbb R^{784}\).
The network is an Equinox ReLU MLP with architecture \(784 \to 100 \to 100 \to 10\), corresponding to three affine weight matrices and two hidden ReLU layers.
It is trained with cross-entropy loss using SGD with learning rate \(10^{-2}\), momentum \(0.9\), minibatch size \(60\), and the stopping criterion of at least \(97\%\) test accuracy.

Training takes about 5 minutes on an NVIDIA T1200 GPU.

For a trained network with weight matrices \(W_1,W_2,W_3\), the trivial 2-norm Lipschitz bound is
\[
  K = \prod_{\ell=1}^3 \left\|W_\ell\right\|_2.
\]

The ECLipsE comparison is computed by a NumPy/CVXPY implementation of the layerwise SDP recursion for activations slope-restricted in \([0,1]\).
The implementation uses the specialization \(m=1/2\), \(p=0\), initializes \(M_0=I\), and solves the per-layer SDPs with SCS using tolerance \(10^{-9}\) and up to \(200{,}000\) iterations.
The final ECLipsE estimate is computed from the largest eigenvalue of \(W_3^\top W_3 M_2^{-1}\).

The gauge-adjusted trivial bound uses the positive homogeneity of ReLU.
For positive diagonal matrices \(D_1,D_2\), the notebook replaces
\[
  (W_1,W_2,W_3)
  \quad\text{by}\quad
  (D_1W_1,\; D_2W_2D_1^{-1},\; W_3D_2^{-1}),
\]
which leaves the realized ReLU network unchanged but changes the product of spectral norms.
The optimization variables are the log-diagonal entries of \(D_1\) and \(D_2\), initialized at zero.
The objective is the squared trivial bound
\[
  \left\|D_1W_1\right\|_2^2
  \left\|D_2W_2D_1^{-1}\right\|_2^2
  \left\|W_3D_2^{-1}\right\|_2^2,
\]
with gradients supplied by JAX \citep{deepmind_deepmind_2020} and minimized using SciPy's \texttt{trust-constr} method with the SR1 Hessian approximation \citep{virtanen_scipy_2020}.

\section{MNIST classification sweep: details}
\label{sec:appendix-classification}
The sweep of \cref{subsec:results-classification} is implemented in \texttt{mnist\_lipschitz\_sweep.py}.
The supplied summary contains six methods at ten target bounds.

\subsection{Methodology}
\label{subsec:appendix-classification-methodology}
\paragraph{Data}
Inputs are the \(60{,}000\) MNIST training images, scaled to \([0, 1]\) and flattened into \(\mathbb R^{784}\), with no further normalization.
Test accuracy is evaluated on the \(10{,}000\) test images after each epoch; the reported accuracy is from the final epoch.

\paragraph{Architectures}
Every network uses hidden width \(100\) and produces \(K - 1 = 9\) output scores for the \(K = 10\) classes (see below).
The quadrature network is \eqref{eq:quadrature-network} with \(L = 2\): \(W \in \mathbb R^{100 \times 784}\), \(A^1, B^1 \in \mathbb R^{100 \times 100}\), \(A^2, B^2 \in \mathbb R^{9 \times 100}\), and its certificate is the trivial bound \eqref{eq:trivial-lipschitz-bound}.
ReLU-EMM is a ReLU network \(784 \to 100 \to 100 \to 9\) with biases, He-normal initialization, and the trivial bound (the product of the three spectral norms) as its certificate.
Sandwich and AOL are \(784 \to 100 \to 100 \to 9\) with the layers of \cref{sec:results}, scaled by \(\sqrt{L^*}\) at the input and output.
CPL and SLL consist of a semi-orthogonal (Cayley) projection \(784 \to 100\), two residual blocks of width \(100\), and a semi-orthogonal projection \(100 \to 9\), the projections again scaled by \(\sqrt{L^*}\).
The quadrature network uses the Gaussian initialization in \cref{sec:method}, with the hidden scales initialized to the identity.

\paragraph{Output scores and cross-entropy}
The softmax is invariant to adding a constant to every logit, so the \(K\)-dimensional logit vector has one direction that the loss cannot see but that consumes Lipschitz budget.
We therefore have every network output \(s \in \mathbb R^{K - 1}\) and form the logits \(z = H s \in \mathbb R^K\), where \(H \in \mathbb R^{K \times (K - 1)}\) is the Helmert basis of the zero-sum subspace: its columns are orthonormal and orthogonal to \(\bm 1\).
Since \(H\) is an isometry, \(\left\|z\right\| = \left\|s\right\|\) and the Lipschitz certificate of the network transfers unchanged to the logits.
The reported losses fall as low as \(10^{-9}\), far below single-precision resolution of the textbook form \(\log \sum_k e^{z_k} - z_y\), which suffers catastrophic cancellation once the target dominates.
We instead evaluate the identical quantity as
\begin{align}
  \mathcal L(z, y) &= \log\del{1 + \sum_{k \neq y} e^{z_k - z_y}} = \operatorname{logaddexp}\del{0, \operatorname{logsumexp}_{k \neq y}\del{z_k - z_y}},
\end{align}
which is accurate to relative precision whenever the loss is small.

\paragraph{Training}
All methods use Adam with batch size \(60\) (\(1{,}000\) steps per epoch) for \(200\) epochs, the learning rate decaying linearly from \(10^{-3}\) to zero.
The baselines minimize the cross-entropy alone.
The EMM methods minimize \(\mathcal L + \mu \hat K(\theta)\) on each minibatch with \(\hat K\) computed by eigendecomposition of the Gram matrices of the weight layers, starting from \(\mu = 10^{-4}\) and updating \(\log \mu \gets \log \mu + \log\del{\hat K(\theta) / L^*}\) every \(200\) steps.

\paragraph{Reported quantities}
The training loss is a geometric exponential moving average of the minibatch loss, \(\overline{\log \mathcal L} \gets 0.999\, \overline{\log \mathcal L} + 0.001 \log \mathcal L\), bias-corrected and reported as \(\exp \overline{\log \mathcal L}\); its time constant is approximately \(1{,}000\) included minibatches.
Each minibatch loss is evaluated at the parameters before that step's update, and for the EMM methods the average is updated only when the certificate at those parameters satisfies \(\hat K \leq L^*\).
Consequently this statistic averages losses over feasible iterates and is not a full-dataset evaluation of a single final network.
The horizontal coordinate in \cref{fig:sweep-summary} is the certificate of the final network; the classification script does not replace an infeasible final network with an earlier feasible iterate.
All final networks satisfy their target bound except Quadrature-EMM at \(L^*=77.4264\), whose final certificate is \(77.6014\), exceeding the target by \(0.23\%\).
(Its tabulated loss still averages only feasible iterates.)
The baselines certify exactly \(L^*\) by construction.

\label{subsec:appendix-classification-results}
\begin{figure}[h]
  \centering
  \includegraphics[width=0.8\linewidth]{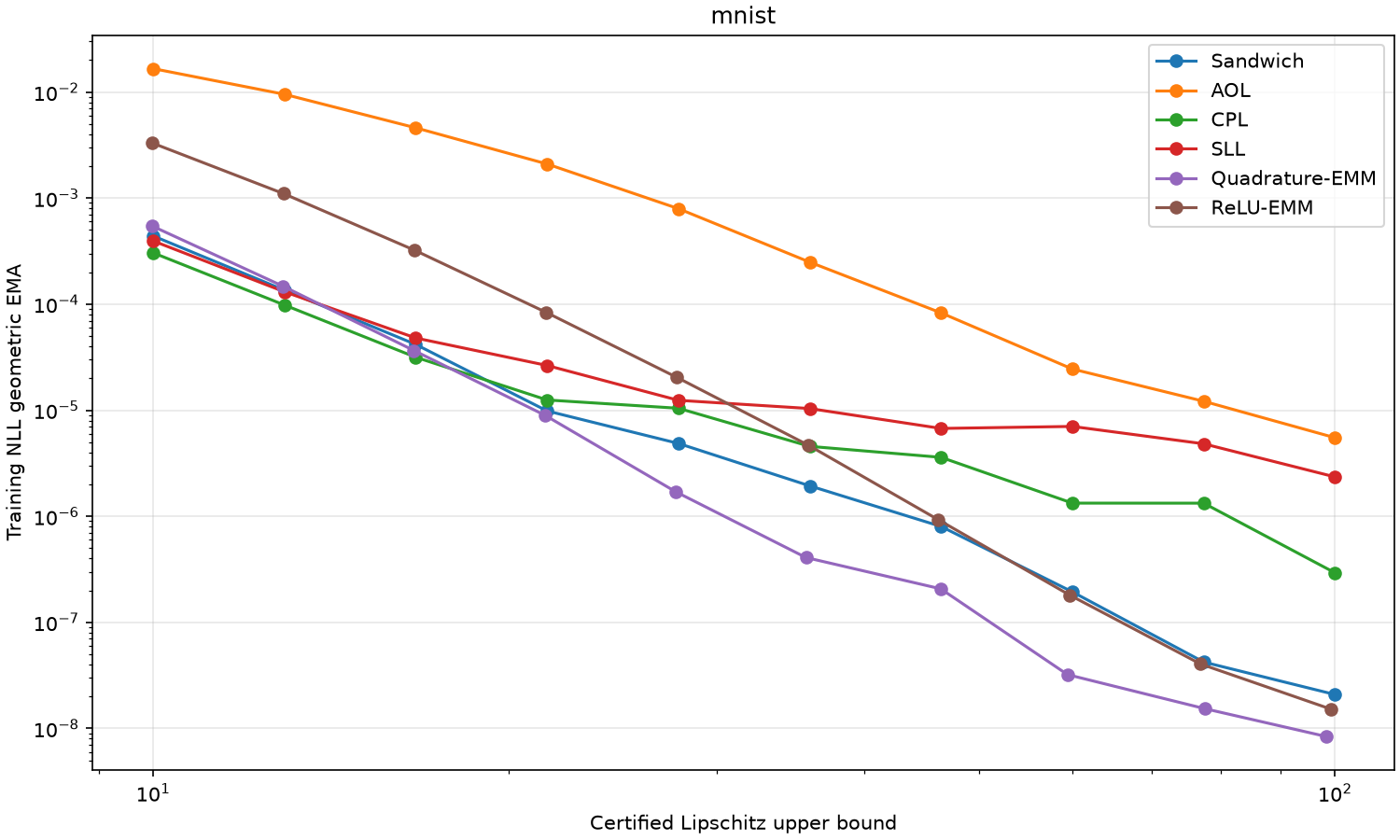}
  \caption{MNIST classification: geometric moving average of training cross-entropy against the final network's certified Lipschitz upper bound, for ten target bounds \(L^*\) log-spaced between \(10\) and \(100\). EMM losses average only feasible iterates. The tabulated values are in \cref{tab:sweep-loss}.}
  \label{fig:sweep-summary}
\end{figure}

\section{MNIST autoencoders: details}
\label{sec:appendix-autoencoder}
The comparison of \cref{subsec:results-autoencoder} is implemented in \texttt{mnist\_autoencoder\_lipschitz\_comparison.py}.

\subsection{Methodology}
\label{subsec:appendix-autoencoder-methodology}
\paragraph{Data and loss}
Inputs and targets are the \(60{,}000\) MNIST training images scaled to \([0, 1]\) and flattened into \(\mathbb R^{784}\); the test set contains \(10{,}000\) images.
The loss is the mean squared error per pixel, \(\mathcal L[f] = \frac{1}{784}\, \mathbb E \left\|f(x) - x\right\|^2\), and reconstructions are not clipped to \([0, 1]\) when the loss is evaluated.
Training and test MSE are evaluated after training by a full pass over the respective set, in batches of \(1{,}000\).

\paragraph{Architectures}
The inner networks use hidden widths \([d,d]\), for \(d \in \{8,16,32,64,128\}\).
The quadrature network is \eqref{eq:quadrature-network} with \(L=2\), \(W\in\mathbb R^{d\times784}\), \(A^1,B^1\in\mathbb R^{d\times d}\), and \(A^2,B^2\in\mathbb R^{784\times d}\), with the Gaussian initialization in \cref{sec:method}.
ReLU-EMM, Sandwich, AOL, FullSort, and MaxMin have three weight layers, \(784\to d\to d\to784\).
ReLU-EMM uses He-normal weights and zero biases.
CPL and SLL have a semi-orthogonal input projection \(784\to d\), two residual blocks of width \(d\), and a semi-orthogonal output projection \(d\to784\).
FullSort sorts all \(d\) coordinates after each hidden affine map, whereas MaxMin sorts groups of two.
Their weight matrices are first divided by their spectral norm and then undergo \(15\) first-order Bj\"orck iterations; the output layer is linear.
Every neural model is wrapped in the learned coordinate transformation described below.
PCA with \(d\) components is fit to the training set by randomized SVD and reconstructs \(x\mapsto\bar x+P^\top P(x-\bar x)\).
It has no learned norm wrapper and is exactly 1-Lipschitz in the Euclidean norm.

\paragraph{Training and reported quantities}
All networks use Adam with batch size \(64\) for \(80\) epochs, with initialization seed \(42\) and shuffling seed \(0\).
Each shuffled epoch drops its final incomplete batch, giving \(937\) updates per epoch and \(74{,}960\) updates in total.
The learning rate decays linearly from \(10^{-3}\) to zero.
Both EMM methods use \(\mu_0=10^{-4}\) and update the multiplier every \(200\) steps.
Adam uses \(\epsilon=0\) for Quadrature-EMM and \(\epsilon=10^{-8}\) for ReLU-EMM and the baselines; the positive value avoids division by zero for dead ReLU units.
The logged geometric moving averages use decay \(0.999\) and are gated on feasibility for EMM, but the tabulated MSE values are full-dataset evaluations.
Both EMM methods report the final iterate if feasible, otherwise the last stored feasible iterate.
All ten reported EMM autoencoders are feasible at the final step \(74{,}960\).

The empirical Lipschitz lower bound searches over pairs \((x,y)\in\mathbb R^{784}\times\mathbb R^{784}\) in each model's norm.
L-BFGS-B minimizes
\[
-\tfrac12\left[\log\bigl(\|f(x)-f(y)\|_U^2+10^{-16}\bigr)
-\log\bigl(\|x-y\|_U^2+10^{-16}\bigr)\right]
\]
for at most \(100\) iterations from three restarts.
Each restart begins with a random pair of test images; independent \(\mathcal N(0,0.1^2)\) noise is added after the first restart.
The search is unconstrained in pixel space and reports the largest distance ratio found.
PCA's lower and upper bounds are both set analytically to one.

\paragraph{Searching for a metric}
The saved configuration has \texttt{learn\_norm=true}.
For each neural model, independently, we learn an upper-triangular matrix \(U\in\mathbb R^{784\times784}\), initialized to the identity.
Its diagonal entries are exponentials of unconstrained parameters, and its strictly upper-triangular entries are unconstrained.
The resulting map and norm are
\begin{align}
f_{\theta,U}(x)&=U g_\theta(U^{-1}x),
&
\|x\|_U&=\|U^{-1}x\|_2,
\label{eq:learned-autoencoder-norm}
\end{align}
with inverse applications evaluated by triangular solves.
For a Euclidean certificate \(\hat K(\theta)\) of \(g_\theta\),
\[
\|f_{\theta,U}(x)-f_{\theta,U}(y)\|_U
\leq \hat K(\theta)\|x-y\|_U.
\]
Thus the wrapper preserves the inner certificate in the learned norm, and the constraint penalizes only the inner network.
The loss remains \(\frac{1}{784}\mathbb E\|f_{\theta,U}(x)-x\|_2^2\); it is not reweighted by \(U^{-1}\).
All neural methods optimize over the same family of norms, but their trained norms need not coincide.
PCA remains a reference for pixel reconstruction quality rather than a comparison under a shared learned metric.

The implementation places no bound on the condition number of \(U\).
In Euclidean coordinates the generic certificate is \(\kappa_2(U)\hat K(\theta)\), where \(\kappa_2(U)=\|U\|_2\|U^{-1}\|_2\).
The learned norm must therefore be specified when interpreting a certificate of one.
This input/output transformation is distinct from the hidden quadrature scales \(\Lambda^\ell\), whose reciprocal logs are parameterized in the zero-sum Helmert subspace as described in \cref{sec:method}.
Those scales preserve the local isometry of the paired cosine/sine activation in Euclidean coordinates and do not enter the trivial bound.

\subsection{Results}
\label{subsec:appendix-autoencoder-results}
\Cref{tab:autoencoder-mse-full} reports training and test pixel MSE side by side.
Test MSE is slightly lower than training MSE for every method and width, including PCA.
Quadrature-EMM achieves the lowest values on both splits at every width.
\Cref{fig:autoencoder-mse} shows the supplied training-MSE summary.
\Cref{fig:autoencoder-reconstructions} shows the corresponding reconstructions at \(d=16\).

\Cref{tab:autoencoder-lipschitz} compares empirical lower bounds and certificates in each model's learned norm.
For Quadrature-EMM, the empirical values range from \(0.5227\) to \(0.7643\), whereas the certificates range from \(0.9724\) to \(0.9876\).
FullSort yields empirical values from \(0.6526\) to \(0.9863\), while the other learned-norm baselines also have substantial gaps.
A low empirical value need not imply a loose certificate, since the search provides only a lower bound.

\begin{table}[h]
\centering
\caption{MNIST autoencoders: training / test pixel MSE in units of \(10^{-3}\). Neural models satisfy the constraint in their respective learned norms. Lowest values on each split in bold.}
\label{tab:autoencoder-mse-full}
\small
\setlength{\tabcolsep}{4pt}
\begin{tabular}{@{}lrrrrr@{}}
\toprule
Method & \(d=8\) & \(d=16\) & \(d=32\) & \(d=64\) & \(d=128\) \\
\midrule
PCA & \(37.87 / 37.44\) & \(27.29 / 26.86\) & \(17.24 / 16.83\) & \(9.29 / 9.05\) & \(4.29 / 4.23\) \\
Sandwich & \(42.01 / 41.65\) & \(30.49 / 30.12\) & \(19.10 / 18.65\) & \(12.02 / 11.71\) & \(7.77 / 7.58\) \\
AOL & \(38.85 / 38.47\) & \(28.61 / 28.19\) & \(20.18 / 19.72\) & \(10.11 / 9.88\) & \(5.11 / 5.04\) \\
CPL & \(37.87 / 37.44\) & \(27.30 / 26.86\) & \(17.25 / 16.84\) & \(9.29 / 9.06\) & \(4.27 / 4.21\) \\
SLL & \(37.87 / 37.44\) & \(27.30 / 26.86\) & \(17.25 / 16.84\) & \(9.29 / 9.05\) & \(4.27 / 4.21\) \\
FullSort & \(37.87 / 37.44\) & \(27.30 / 26.86\) & \(17.26 / 16.85\) & \(9.39 / 9.16\) & \(6.59 / 6.46\) \\
MaxMin & \(37.87 / 37.44\) & \(27.29 / 26.86\) & \(17.24 / 16.83\) & \(9.28 / 9.05\) & \(4.27 / 4.21\) \\
ReLU-EMM & \(37.87 / 37.44\) & \(28.29 / 27.88\) & \(19.09 / 18.63\) & \(10.61 / 10.34\) & \(6.02 / 5.89\) \\
Quadrature-EMM & \(\bm{33.37 / 33.05}\) & \(\bm{22.66 / 22.23}\) & \(\bm{13.28 / 12.99}\) & \(\bm{7.03 / 6.90}\) & \(\bm{3.38 / 3.36}\) \\
\bottomrule
\end{tabular}
\end{table}

\begin{table}[h]
\centering
\caption{MNIST autoencoders: empirical lower bound / certified upper bound in each neural model's learned norm. PCA uses the Euclidean norm. Bounds in different rows refer to different norms.}
\label{tab:autoencoder-lipschitz}
\small
\setlength{\tabcolsep}{3pt}
\begin{tabular}{@{}lrrrrr@{}}
\toprule
Method & \(d=8\) & \(d=16\) & \(d=32\) & \(d=64\) & \(d=128\) \\
\midrule
PCA & \(1.0000 / 1.0000\) & \(1.0000 / 1.0000\) & \(1.0000 / 1.0000\) & \(1.0000 / 1.0000\) & \(1.0000 / 1.0000\) \\
Sandwich & \(0.1104 / 1.0000\) & \(0.1754 / 1.0000\) & \(0.2231 / 1.0000\) & \(0.3593 / 1.0000\) & \(0.3905 / 1.0000\) \\
AOL & \(0.1642 / 1.0000\) & \(0.1318 / 1.0000\) & \(0.1982 / 1.0000\) & \(0.3454 / 1.0000\) & \(0.3427 / 1.0000\) \\
CPL & \(0.0932 / 1.0000\) & \(0.1602 / 1.0000\) & \(0.2747 / 1.0000\) & \(0.4436 / 1.0000\) & \(0.4614 / 1.0000\) \\
SLL & \(0.1095 / 1.0000\) & \(0.1482 / 1.0000\) & \(0.2499 / 1.0000\) & \(0.3788 / 1.0000\) & \(0.5066 / 1.0000\) \\
FullSort & \(0.6526 / 1.0000\) & \(0.7126 / 1.0000\) & \(0.9490 / 1.0000\) & \(0.9376 / 1.0000\) & \(0.9863 / 1.0000\) \\
MaxMin & \(0.3989 / 1.0000\) & \(0.2980 / 1.0000\) & \(0.1896 / 1.0000\) & \(0.3930 / 1.0000\) & \(0.5278 / 1.0000\) \\
ReLU-EMM & \(0.2954 / 0.9809\) & \(0.2471 / 0.9846\) & \(0.2482 / 0.9915\) & \(0.4333 / 0.9946\) & \(0.5119 / 0.9917\) \\
Quadrature-EMM & \(0.5946 / 0.9724\) & \(0.5227 / 0.9770\) & \(0.6028 / 0.9831\) & \(0.7643 / 0.9876\) & \(0.6815 / 0.9864\) \\
\bottomrule
\end{tabular}
\end{table}

\begin{figure}[h]
  \centering
  \includegraphics[width=0.85\linewidth]{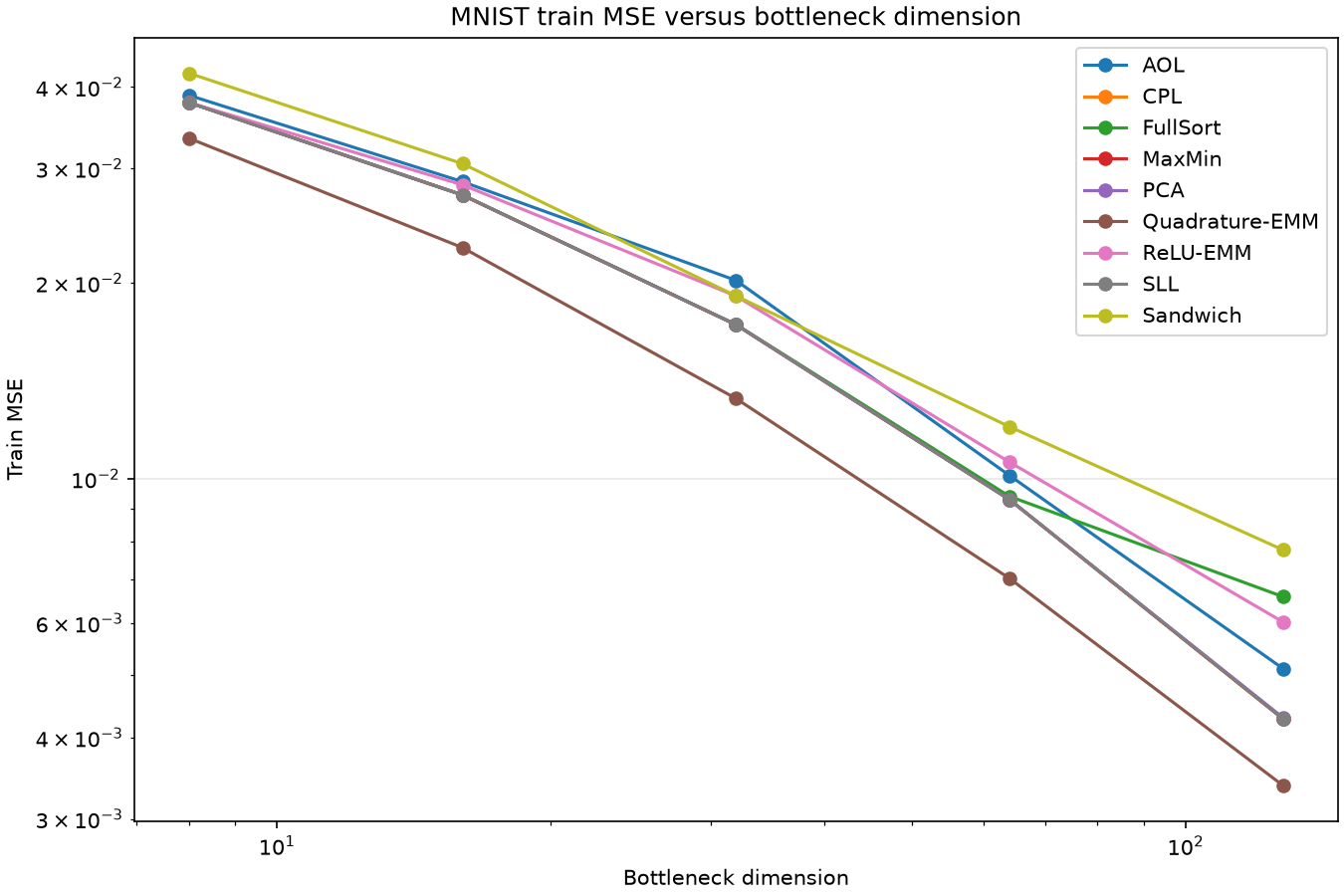}
  \caption{MNIST autoencoders: training pixel MSE against bottleneck dimension, using the supplied results for all eight neural methods and PCA. Each neural model has a Lipschitz bound at most one in its own learned norm.}
  \label{fig:autoencoder-mse}
\end{figure}

\begin{figure}[p]
  \centering
  \includegraphics[width=0.85\linewidth]{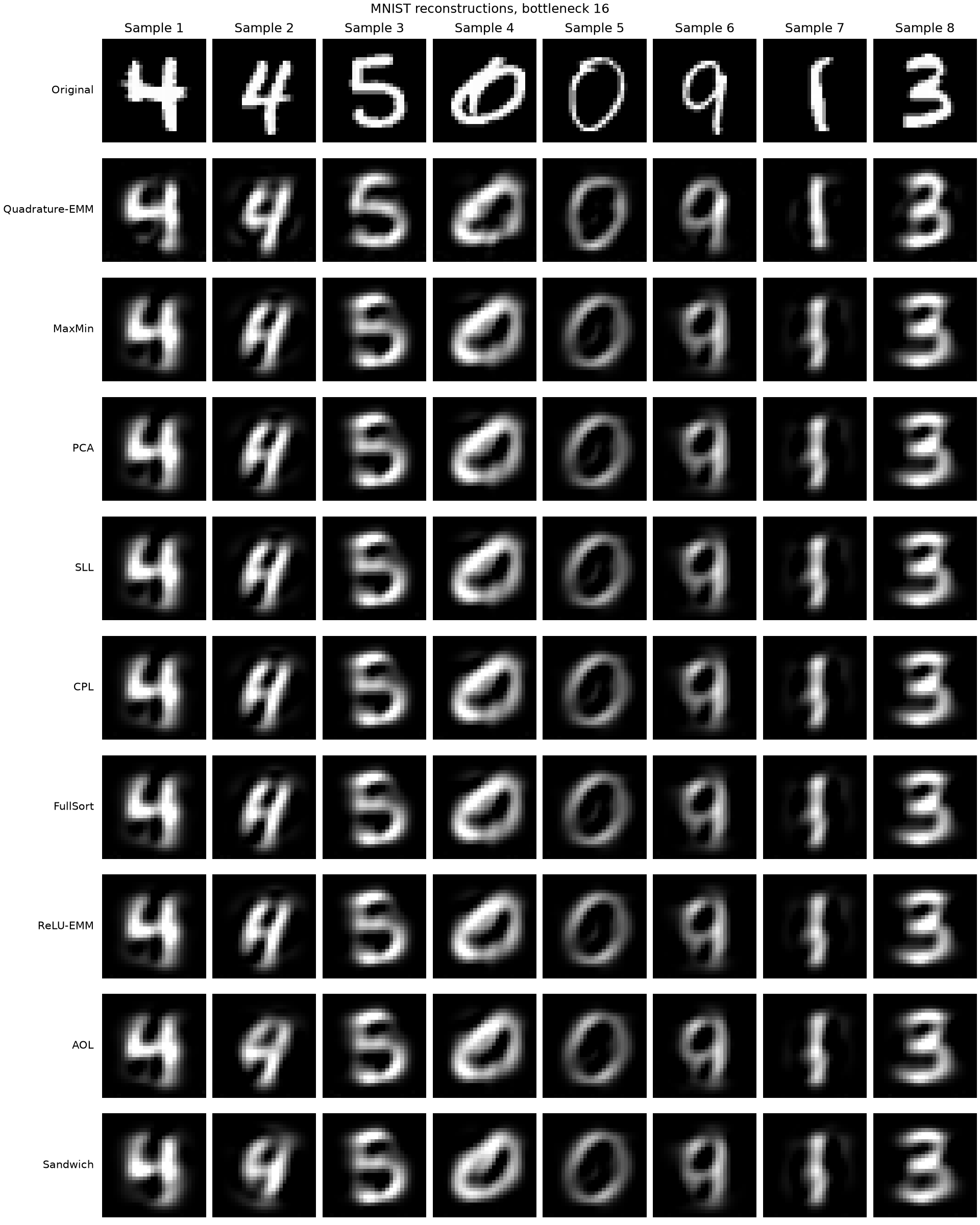}
  \caption{MNIST autoencoders at bottleneck dimension \(d=16\): eight test images (top row) and reconstructions from all eight neural methods and PCA, ordered by full-test-set pixel MSE. Neural models are constrained in their own learned norms.}
  \label{fig:autoencoder-reconstructions}
\end{figure}

\end{document}